\RequirePackage{rotating}
\documentclass[smallcondensed]{svjour3}   

\smartqed 
\usepackage{verbatim,amscd,graphicx,booktabs,multirow,graphics,url}
\usepackage[table]{xcolor}
\usepackage{algorithm}
\usepackage{float}
\usepackage{breakurl}
\RequirePackage{rotating}
\usepackage{adjustbox}
\usepackage{footnote}
\usepackage{amsmath}
\usepackage{slashbox} 
\usepackage{amssymb,amsfonts}

\usepackage{amsthm}
\usepackage{lipsum}
\usepackage{paralist}
\usepackage{array}
  \usepackage{graphics} 
  \usepackage{epsfig}
\usepackage{graphicx}  \usepackage{epstopdf}
 \usepackage[colorlinks=true]{hyperref}
\hypersetup{urlcolor=blue, citecolor=red}

\DeclareMathOperator{\sgn}{sgn}

\newcolumntype{B}{>{\color{black}}c}

\DeclareMathOperator{\diag}{diag}  

\usepackage{amsthm,amssymb,amsfonts}
\usepackage{amsopn,amstext}
\usepackage{bm}
\usepackage[graphicx]{realboxes}
\usepackage{float}
\usepackage{subcaption}
\usepackage{mathrsfs}
\usepackage{placeins}
\usepackage{chemformula}
\usepackage{listing}
\usepackage{xcolor}
\usepackage{multirow}
\usepackage{mathtools, eucal}
\usepackage[noend]{algpseudocode}
\usepackage[justification=centering]{caption}
\usepackage[bottom]{footmisc}
\usepackage{nccmath}
\usepackage{arydshln}
\usepackage{comment}
\usepackage{pdflscape}
\usepackage{slashbox}
\usepackage{color}
\usepackage{makecell}

\newcommand{\argmin}{\mbox{argmin}}

\newcommand{\lp}{(}
\newcommand{\rp}{)}
\newcommand{\R}{\mathbb{R}}

\newcommand{\ve}{\mbox{vec}}
\newcommand{\hv}{\mbox{hvec}}
\newcommand{\nct}{\frac{n(n+1)}{2}}
\newcommand{\RNum}[1]{\uppercase\expandafter{\romannumeral #1\relax}}
\newcommand{\ho}[0]{\color{black}}

\begin{document}
\title{Quadratic Surface Twin Support Vector Machine for Imbalanced Data 
 \thanks{
H.~Moosaei and M.~Hlad\'{\i}k were supported by the Czech Science Foundation Grant 25-15714S, 
 and 
 H.~Fu by the National Natural Science Foundation of China Grant 72201052.}
}
\titlerunning{Kernel-Free Universum Quadratic Surface Twin Support Vector Machines for Imbalanced Data}        

\author{Hossein Moosaei
\and Milan Hlad\'{i}k  
\and \mbox{Ahmad Mousavi}$^*$\and Zheming Gao \and Haojie Fu
}

\institute{$*$ Corresponding Author \at
  \and
  Hossein Moosaei \at
  Department of Informatics, Faculty of Science, Jan Evangelista Purkyně University, \'{U}st\'{i} nad Labem, Czech Republic\\
              \email{hmoosaei@gmail.com, hossein.moosaei@ujep.cz}   
    \and 
     Milan Hlad\'{i}k  \at
  Department of Applied Mathematics, Faculty  of  Mathematics  and  Physics, Charles University, Prague, Czech Republic \\
  \email{hladik@kam.mff.cuni.cz}    
    \and
    Ahmad Mousavi\at
  Department of Mathematics and Statistics, American University, Washington, DC, USA \\
  \email{mousavi@american.edu}
 \and
  Zheming Gao \at
  Department of Industrial and Systems Engineering, University of North Carolina at Charlotte, USA\\
  \email{zgao7@charlotte.edu}
 \and
    Haojie Fu \at
  College of Information Science and Engineering, Northeastern University, Shenyang, Liaoning 110819, China\\
  \email{haojie12325@163.com}
}
\maketitle

\begin{abstract}
Binary classification tasks with imbalanced classes pose significant challenges in machine learning. Traditional classifiers often struggle to accurately capture the characteristics of the minority class, resulting in biased models with subpar predictive performance. In this paper, we introduce a novel approach to tackle this issue by leveraging Universum points to support the minority class within quadratic twin support vector machine models. Unlike traditional classifiers, our models utilize quadratic surfaces instead of hyperplanes for binary classification, providing greater flexibility in modeling complex decision boundaries. By incorporating Universum points, our approach enhances classification accuracy and generalization performance on imbalanced datasets. We generated four artificial datasets to demonstrate the flexibility of the proposed methods. Additionally, we validated the effectiveness of our approach through empirical evaluations on benchmark datasets, demonstrating superior performance compared to conventional classifiers and existing methods for imbalanced classification.
\end{abstract}

\section{Introduction} \label{sec: introduction}
The support vector machine (SVM) model marked a pivotal advancement in binary classification, achieving this by identifying two parallel supporting hyperplanes that maximize the minimum margin between classes through a convex quadratic programming problem \cite{vapnik1974theory}. Weston et al. introduced the Universum support vector machine ($\mathfrak{U}$-SVM), which enhances the model by incorporating prior information \cite{weston2006inference}. This is achieved by including Universum samples that do not belong to any class. Additionally, the Twin Support Vector Machine (TSVM) represents a significant extension, utilizing two non-parallel hyperplanes to simultaneously separate the classes. This innovative approach maximizes the margin between them while ensuring sufficient coverage of both majority and minority class instances.

However, many real-world applications pose challenges, as they often involve data that is not linearly separable. Addressing this requires robust techniques capable of effectively handling nonlinear situations. One well-established approach is to leverage the concept of kernel methods, which relies on the existence of a nonlinear mapping. This mapping transforms the original data into a higher-dimensional, potentially infinite-dimensional feature space, where the data becomes linearly separable. Notably, the appeal of kernel methods lies in their ability to bypass the need for explicitly knowing this mapping. Instead, they focus on selecting appropriate kernel functions with advantageous properties. However, the challenge lies in determining the suitable kernel function, and fine-tuning the associated hyperparameters can be computationally demanding. As a result, there is a growing interest in practical methods that seek nonlinear classifiers directly in the original feature space.

Kernel-free models offer a direct approach for handling nonlinearly separable datasets without the need to map them to a higher-dimensional feature space. The quadratic surface support vector machine (QSSVM) \cite{dagher2008quadratic} utilizes quadratic surfaces for class separation. Bai et al.\cite{bai2015quadratic} introduced a kernel-free least squares QSSVM tailored for disease diagnosis. A kernel-free support vector machine, utilizing a double-well potential proposed in \cite{gao2021kernel}, aims to achieve a specialized fourth-order polynomial separating surface. However, when pursuing nonlinear classifiers, the complexity of the optimization program significantly increases due to the involvement of numerous decision variables, leading to computational challenges and potential overfitting. To mitigate this, incorporating a surrogate that promotes sparsity among decision variables proves advantageous \cite{moosaei2023sparse,mousavi2022quadratic,mousavi2020survey}. To achieve less expensive quadratic programs, \cite{gao2019quadratic} discusses a least squares twin QSSVM, leveraging two quadratic surfaces for data representation.

Class imbalance in binary classification presents numerous challenges that can significantly impact the performance and reliability of machine learning models \cite{haixiang2017learning}. When one class dominates the dataset, classifiers tend to exhibit bias towards the majority class, resulting in subpar predictions for the minority class. This imbalance can lead to various issues, including reduced sensitivity and specificity, misinterpretation of model accuracy, and distorted decision boundaries. In critical domains such as fraud detection or medical diagnoses, where the minority class represents rare events, the consequences of misclassification can be severe \cite{fotouhi2019comprehensive,makki2019experimental}. Failing to address class imbalance may result in suboptimal model performance, compromised decision-making, and heightened risks of costly errors or missed opportunities. Thus, it is imperative to implement effective strategies to mitigate class imbalance and ensure that machine learning models accurately capture the underlying data patterns, irrespective of class distribution.

{\ho
{\paragraph{Gap in the literature.}  
Existing methods either (i) rely on kernel mappings with high computational cost, (ii) do not explicitly address class imbalance, or (iii) use quadratic surfaces without effectively incorporating additional prior knowledge (such as Universum points). As a result, there is a lack of models that simultaneously handle nonlinear decision boundaries, class imbalance, and incorporate prior knowledge in a computationally efficient framework.}}

{ Building on this gap, our} methodology draws inspiration from the approach introduced in \cite{richhariya2020reduced}, which incorporates Universum points to bolster the minority class in linearly separable datasets. Building upon this foundation, we extend its application to the framework of quadratic twin support vector machines (QTSVM). Leveraging the inherent flexibility of QTSVM, which employs quadratic surfaces instead of hyperplanes for binary classification, we aim to enhance the efficacy of Universum points in addressing class imbalance. This innovative integration enables us to capture intricate class relationships and adapt to complex data distributions more effectively. In addition to the hinge loss, we derive the least squares formulation, which adopts the $\ell_2$ norm to penalize slack variables and replaces inequality and equality constraints. To validate the effectiveness of our proposed methods, we conduct experiments on various datasets, including artificial and public benchmark datasets. These experiments demonstrate the efficacy and efficiency of our proposed models compared to other well-known SVM models.

{\ho
In summary, the main contributions of this paper are as follows:
\begin{itemize}
    \item We introduce the Imbalanced Universum Quadratic Twin Support Vector Machine (Im-$\mathfrak{U}$-QTSVM), which leverages quadratic surfaces instead of hyperplanes to better capture complex decision boundaries in binary classification.
    \item We incorporate Universum points to guide the minority class, improving generalization and reducing bias toward the majority class in imbalanced datasets.
    \item We propose both hinge-loss and least squares formulations, enabling either standard QP solutions or closed-form solutions for computational efficiency.
     \item We establish the theoretical properties and computational characteristics of the proposed models, including guaranteed optimality, uniqueness under mild conditions, sufficient conditions for nonzero solutions, and an analysis of computational complexity.

    \item We demonstrate the effectiveness of the proposed methods through experiments on both artificial and benchmark datasets, showing superior performance compared to other existing approaches.
\end{itemize}
}
\paragraph{Notation.}
$ \mathbb R^n$ denotes the $ n$-dimensional real vector space. For $x\in \mathbb R^n$, the notation $|x|$ signifies element-wise absolute value. $A^T$ represents the transpose of a matrix $A$ and $\|\cdot\|$ denotes the Euclidean norm.
Consider a real-valued function $f$ defined on $\mathbb R^n$. The gradient of $f$ at a point $x$ is denoted by the $n$-dimensional column vector $\nabla f(x)$.
Let $S_n$ denote the set of all $n\times n$ symmetric matrices. For a matrix $A\in S_n$, the notation $A\succ 0$ indicates that $A$ is positive definite.

\section{Problem Statement and Related Works} \label{sec: related_works}

We start by describing the fundamental classical and kernel-free models proposed in the literature for binary classification when dealing with an (almost) linearly separable data set. 

\subsection{Universum Twin  Support Vector Machine}
To optimize binary classification by identifying two hyperplanes that are maximally distant from each other and as close as possible to their respective classes, Jayadeva et al.~\cite{khemchandani2007twin} proposed the Twin Support Vector Machine (TSVM). This model delineates two non-parallel hyperplanes, as illustrated below:
\begin{equation*}
w_{1}^{T}x+b_{1}=0,\quad \mbox{and} \quad w_{2}^{T}x+b_{2}=0,
\end{equation*}
where $w_{1}, w_{2 } \in \mathbb{R}^{n}$ and  $b_{1}, b_{2 }\in \mathbb{R}$.
Assume that the data points belonging to class $+1$ and class $-1$ are represented by the rows of matrices $A \in \mathbb{R}^{m_{1} \times n}$ and $B \in \mathbb{R}^{m_{2} \times n}$, respectively. Solving the following two quadratic programming problems yields the TSVM classifiers: 
\begin{equation}\label{tsvm}
\begin{aligned}
\min_{w_{1},b_{1},\xi_{1}} \quad & \frac{1}{2} \left\|A w_{1}+e_{1}b_{1} \right\|^2 + \frac{C_{1}}{2} e_{2}^{T}\xi_{1} \\
\textrm{s.t.} \quad & -( B w_{1}+e_{2}b_{1})+\xi_{1} \ge  e_{2}, \nonumber\ \
 \xi_{1} \ge  0,
\\ \\
\min_{w_{2},b_{2},\xi_{2}} \quad & \frac{1}{2} \left\|B w_{2}+e_{2}b_{2} \right\|^2 + \frac{C_{2}}{2} e_{1}^{T}\xi_{2} \\
\textrm{s.t.} \quad & ( A w_{2}+e_{1}b_{2})+\xi_{2} \ge  e_{1}, \ \
 \xi_{2} \ge  0,
\end{aligned}
\tag{TSVM}
\end{equation}
where $ C_{1} $,  $ C_{2} > 0 $ are  penalty parameters,  $ \xi _{1}  $, $ \xi _{2}  $ are slack vectors, and $ e_{1} $, $ e_{2  }$ are vectors of ones  of appropriate dimension.

Universum data is defined as a set of unlabeled samples that do not belong to any specific class \cite{sinz2007analysis,weston2006inference}. This type of data can encode past knowledge by providing meaningful information within the same domain as the problem at hand. Incorporating Universum data has been shown to effectively improve learning performance in both classification and clustering tasks. By integrating Universum data into classical models, researchers have enhanced the computational efficiency and generalization ability of various existing methods \cite{moosaei2023universum,moosaei2023sparse,weston2006inference,xiao2021new}. 

Let $U\in \mathbb{R}^{r\times m}$ be a matrix where each row represents a Universum point. 
The Universum Twin  Support Vector Machine ($ \mathfrak{U}$-TSVM) was proposed to enhance the classification performance of (\ref{tsvm}) \cite{qi2012twin}. The $\mathfrak{U}$-TSVM was constructed by incorporating Universum data in the TSVM model,
as the following  pair of quadratic programming problems (QPPs):
  \begin{equation}\tag{$\mathfrak{U}$-TSVM}
  \begin{aligned}\label{utsvm} 
\mathop {\min }_{w_{1},b_{1},\xi_{1},\psi_{1} \,\,} & \frac{1}{2}{{\left\|A w_{1}+e_{1}b_{1} \right\|}^{2}}+\dfrac{C_{1}}{2}e_{2}^{T}\xi_{1}+\frac{C_{u}}{2} e_{u}^{T}\psi_{1}\nonumber \\
\textrm{s.t.} \quad& -( B w_{1}+e_{2}b_{1})+{{\xi }_{1}}\ge  e_{2},\nonumber \\ 
 & ( Uw_{1}+e_{u}b_{1} )+\psi_{1}\ge   (-1+\varepsilon)e_{u}, \nonumber \\ 
 & \xi _{1} ,  \psi_{1}\ge  0,\nonumber
\\ \\
\mathop {\min }_{w_{2},b_{2},\xi_{2},\psi_{2} \,\,} & \frac{1}{2}{{\left\|B w_{2}+e_{2}b_{2} \right\|}^{2}}+\dfrac{C_{2}}{2}e_{1}^{T}\xi_{2}+\frac{C_{u}}{2} e_{u}^{T}\psi_{2}\nonumber \\ 
\textrm{s.t.} \quad & ( A w_{2}+e_{1}b_{2})+{{\xi }_{2}}\ge  e_{1},\nonumber \\ 
 & -( Uw_{2}+e_{u}b_{2} )+\psi_{2}\ge   (-1+\varepsilon)e_{u}, \nonumber\\ 
 & \xi _{2},  \psi_{2}\ge  0,\nonumber
\end{aligned}
\end{equation}
where $C_1, C_2,$ and $C_u$ are positive penalty parameters, $\varepsilon\in (0, 1)$ is the tolerance value for the Universum class, and $\xi_{1}, \xi_{2}, \psi_{1},$ and $\psi_{2}$ are measures of the violation of the associated constraints. We may derive their dual problems by applying Lagrangian functions, as shown below.
\begin{eqnarray*}
\mathop {\max }_{\alpha, \beta} && -\frac{1}{2}(G^{T}\alpha-O^{T}\beta)^{T}(H^{T}H)^{-1}(G^{T}\alpha-O^{T}\beta)+e_{2}^{T}\alpha+(\varepsilon-1)e_{u}^{T}\beta  \\
\textrm{s.t.} && 0\leq\alpha\leq C_{1}e_{2},\nonumber \ \ 
  0\leq \beta\leq C_{u}e_{u},
\end{eqnarray*}
and
\begin{align*}
\mathop {\max }_{\alpha^{*}, \beta^{*}} &\  -\frac{1}{2}(H^{T}\alpha^{*}-O^{T}\beta^{*})^{T}(G^{T}G)^{-1}(H^{T}\alpha^{*}-O^{T}\beta^{*})+e_{1}^{T}\alpha^{*}+(\varepsilon-1)e_{u}^{T}\beta^{*} \\
\textrm{s.t.} & ~~0\leq\alpha^{*}\leq C_{2}e_{1},\nonumber \ \ 
 0\leq \beta^{*}\leq C_{u}e_{u},
\end{align*}
where $H=[A~~e_{1}], G=[B~~e_{2}] $, and $O=[U~~e_{u}]$. Once the above dual QPPs 
are solved, we can obtain the following parameters:
\begin{equation*}
\begin{bmatrix}
w_{1}\\b_{1}
\end{bmatrix}=-(H^{T}H)^{-1}(G^{T}\alpha-O^{T}\beta),
\quad \mbox{and} \quad 
\begin{bmatrix}
w_{2}\\b_{2}
\end{bmatrix}=(G^{T}G)^{-1}(H^{T}\alpha^{*}-O^{T}\beta^{*}).
\end{equation*}
Once these vectors
are obtained, the separating planes 
$w_{1}^{T}x+b_{1}=0$ and $w_{2}^{T}x+b_{2}=0,$
are known. 
A new data point $ x \in \mathbb{R}^{n} $ is allocated to class $i\in\{+1,-1\}$ by a rule similar to the TSVM.

\subsection{Least-Square Twin Support Vector Machine}
To improve the computational efficiency of a classifier, the Least-Square Twin Support Vector Machine (LS-TSVM)  \cite{kumar2009least} is introduced, inspired by TSVM. Unlike TSVM, LS-TSVM employs equality constraints rather than inequality constraints, leading to the solution of only a pair of linear equations. Typically, the linear LS-TSVM deals with the following pair of quadratic programming problems (QPPs) 

\begin{equation}
\begin{aligned} \label{ls-tsvm} 
 \mathop {\min }\limits_{w_{1},b_{1}\xi_{1}} \ & \frac{1}{2}{{\left\| Aw_{1}+e_{1}b_{1}\right\|}^{2}}+ \dfrac{C_{1}}{2}\|\xi_{1}\|^{2}\nonumber\\
\textrm{s.t.}~ & 
-\left( Bw_{1}+e_{2}b_{1} \right)+\xi_{1}=e_{2},\\
&\nonumber \\ 
\mathop {\min }\limits_{w_{2},b_{2},\xi_{2}} \ & \frac{1}{2}{{\left\| Bw_{2}+e_{2}b_{2}\right\|}^{2}}+ \dfrac{C_{2}}{2}\|\xi_{2}\|^{2} \nonumber\\
\textrm{s.t.}~ & 
\left( Aw_{2}+e_{1}b_{2} \right)+\xi_{2}= e_{1}, 
    \end{aligned} 
    \tag{LS-TSVM}
\end{equation}
where $C_{1}$ and $C_{2}$ are positive penalty parameters.  These convex programs have closed-form solutions. 

Substituting the equality constraints of the above problems into the associated objective functions, one obtains convex unconstrained optimization problems.
Thus, to find their optimum solutions, we set their gradients with respect to $w_{1},b_{1},w_{2}$ and $b_{2}$ equal to $0$. Consequently, we can find the two non-parallel hyperplanes by solving the following systems of linear equations
\begin{align}                       
\begin{bmatrix}
w_{1}\\
b_{1}
\end{bmatrix}  =-(H^{T}H+C_{1}G^{T}G)^{-1} (C_{1}G^{T}e_{2}),\notag 
\\ 
\begin{bmatrix}
w_{2}\\
b_{2}
\end{bmatrix}  =(G^{T}G+C_{2}H^{T}H)^{-1} (C_{2}H^{T}e_{1}), \notag 
\end{align}
using the same notation as before. 
A new data point $x \in \mathbb{R}^{n}$ is assigned to class one or two using a rule similar to that of the TSVM.

\subsection{Least-Square Universum Twin Support Vector Machine}
To enhance the generalization performance of LS-TSVM, Xu et al. \cite{xu2016least} introduced the incorporation of Universum data, leading to the development of the Least-Square Universum Twin Support Vector Machine (LS-$\mathfrak{U}$-TSVM). This model is formulated as follows:
\begin{equation}
\begin{aligned} \label{ls-u-tsvm} 
\mathop {\min }_{w_{1},b_{1},\xi_{1},\psi_{1} \,\,} &\frac{1}{2}{{\left\|A w_{1}+e_{1}b_{1} \right\|}^{2}}+\dfrac{C_{1}}{2}e_{2}^{T}\xi_{1}+\frac{C_u}{2} e_{u}^{T}\psi_{1}\nonumber \\
\textrm{s.t.} \quad & -( B w_{1}+e_{2}b_{1})+{{\xi }_{1}} = e_{2},\nonumber \\ 
 & ( Uw_{1}+e_{u}b_{1} )+\psi_{1}=  (-1+\varepsilon)e_{u},
\\ \\
\mathop {\min }_{w_{2},b_{2},\xi_{2},\psi_{2} \,\,} & \frac{1}{2}{{\left\|B w_{2}+e_{2}b_{2} \right\|}^{2}}+\dfrac{C_{2}}{2}e_{1}^{T}\xi_{2}+\frac{C_{u}}{2} e_{u}^{T}\psi_{2}\nonumber \\
\textrm{s.t.} \quad & ( A w_{2}+e_{1}b_{2})+{{\xi }_{2}}= e_{1},\nonumber \\ 
 & -( Uw_{2}+e_{u}b_{2} )+\psi_{2}= (-1+\varepsilon)e_{u}, 
    \end{aligned} 
    \tag{LS-$\mathfrak{U}$-TSVM}
\end{equation}
where $C_{1}$ and $C_{2}$ are positive penalty parameters, $\varepsilon\in (0, 1)$ is the tolerance value for the Universum class, and $\xi_{1}, \xi_{2}, \psi_{1}$ and $\psi_{2}$ are measures of the violation of associated constraints. 
The unique solution to these problems is closed-form, which is advantageous for handling large-scale applications. By employing the same solution method as LS-TSVM, we can obtain the two non-parallel hyperplanes by solving the following system of linear equations.
\begin{align*}          
\begin{bmatrix}
w_{1}\\
b_{1}
\end{bmatrix} 
=-\big(H^{T}H+C_{1}G^{T}G+C_{u}&O^{T}O\big)^{-1}
  \big(C_{1}G^{T}e_{2}+C_{u}(1-\varepsilon)O^{T}e_{u}\big),\\
  \begin{bmatrix}
w_{2}\\
b_{2}
\end{bmatrix} 
=\big(G^{T}G+C_{2}H^{T}H+C_{u}&O^{T}O\big)^{-1}
 \big(C_{2}H^{T}e_{1}+C_{u}(1-\varepsilon)O^{T}e_{u}\big).
\end{align*} 
Here, we use the same notation as for $\mathfrak{U}$-TSVM. 
A new data point $x \in \mathbb{R}^{n}$ is assigned to class $i \in \{+1, -1\}$ using a rule similar to that of the TSVM.

\subsection{
Quadratic Twin Support Vector Machine
}
Linear models, while effective for linearly separable data, may face limitations when dealing with more complex datasets. To address this challenge, kernel methods have been introduced. However, these methods have their drawbacks \cite{gao2019quadratic}. In response, many quadratic kernel-free methods have been studied recently \cite{mousavi2022quadratic,gao2021novel,gao2019rescaled}. The idea behind these models is analogous to that of their linear version, except they capture quadratic surfaces. Below, we elaborate on the Quadratic Kernel-Free Twin Support Vector Machine (QTSVM) model proposed by Gao et al. \cite{gao2019quadratic}. 

{ Suppose we are given a dataset of $n$ points 
$\{ (x_i, y_i) \}_{i=1}^n$, 
where $x_i $ are the feature vectors and $y_i \in \{-1, +1\}$ are the binary class labels. 
We can partition the indices of the data points based on their labels as follows:
$\mathcal{I}_1:= \{i \in \{1, \dots, n\} \ | \ y_i = 1 \}$ and $\mathcal{I}_2:= \{i \in \{1, \dots, n\} \ | \ y_i = -1 \}$.
}
Suppose we have two classes that include $|\mathcal{I}_1|$ and $|\mathcal{I}_2|$ points, respectively. Then, the formulation of the QTSVM model for binary classification becomes the following:
\begin{equation} \tag{QTSVM}
\label{QTSVM}
\begin{aligned}
\min_{ W_1, b_1,c_1, \xi_1
} \ \ & \sum_{i\in \mathcal{I}_1}(\frac{1}{2}x_i^TW_1x_i+b_1^Tx_i+c_1)^2+
C_1 \sum_{i\in \mathcal{I}_{2}} {\xi_1}_i\\
\textrm{subject to} \quad & 
1+(\frac{1}{2}{x_i^TW_1 x_i + b_1^Tx_i+c_1)\le{\xi_1}_i, \quad  \forall i\in \mathcal{I}_{2}}\\
&  W_1\in \mathbb S^n, b_1\in \mathbb R^n, c_1 \in \mathbb R, \xi_1\in \mathbb R^{|\mathcal{I}_{2}|}_+.
\\ \\
\min_{ W_2, b_2,c_2, \xi_2
} \ \ & \sum_{i\in \mathcal{I}_2}(\frac{1}{2}x_i^TW_2x_i+b_2^Tx_i+c_2)^2+
C_2 \sum_{i\in \mathcal{I}_1} {\xi_2}_i\\
\textrm{subject to} \quad & 
1-(\frac{1}{2}{x_i^TW_2 x_i + b_2^Tx_i+c_2)\le{\xi_2}_i, \quad  \forall i\in \mathcal{I}_{1}}\\
&  W_2\in \mathbb S^n, b_2\in \mathbb R^n, c_2 \in \mathbb R, \xi_2\in \mathbb R^{|\mathcal{I}_{1}|}_+.
\end{aligned}
\end{equation}
where $C_1$ and $C_2$ are positive hyperparameters.

It is important to note that the quadratic-based models are not in the standard form of quadratic programs. Therefore, we will introduce some notations and provide several definitions to facilitate their conversion into standard quadratic program forms. We shall be succinct here, as these definitions are mainly borrowed from \cite{mousavi2019quadratic}. 
For a square matrix $  A = [a_{ij}]_{i=1,\dots,n;j=1,\dots,n}\in \mathbb R^{n\times n}$, its vectorization is given by
$$\ve(A) := \left[a_{11},\dots, a_{n1},a_{12},\dots,a_{n2},\dots,a_{1n},\dots,a_{nn}\right]^T \in \R^{n^2}.$$
In case $A$ is symmetric, $\ve(A)$ contains redundant information, so we often consider its half-vectorization  given by:
$$\hv(  A) := \left[a_{11},\dots, a_{n1},a_{22},\dots,a_{n2},\dots,a_{nn}\right]^T \in \R^{\nct}.$$
Given $n\in \mathbb N$,  { there exists a unique} elimination matrix $  L_n \in \R^{\nct \times n^2}$ such that \cite{magnus1980elimination}
$$  L_n \ve(A) = \hv(A);\quad \forall \,   A \in \mathbb S^n,$$ 
and furthermore, this elimination matrix $  L_n$ has full row rank \cite{magnus1980elimination}. Conversely, for any $n\in \mathbb N$, there is a unique duplication matrix $  D_n \in \R^{n^2 \times \nct}$ such that 
$$  D_n \hv(A) = \ve(A); \quad \forall\,   A \in  \mathbb S^n \qquad \mbox{and} \qquad L_nD_n=I_\nct.$$ 

\begin{definition} \label{def: definitios}
For fixed $i\in [m]$ and $j\in [r]$,  let
\begin{eqnarray}\nonumber
s_i &:=&   \frac{1}{2}\hv( x_i x_i^T), \qquad
s_j :=   \frac{1}{2}\hv(u_ju_j^T), \qquad
r_i :=  [s_i;x_i], \quad
r_j:=  [s_j;u_j],\\
w &:=& \hv(W), \quad
z: =  [w;b], \quad 
V:= \big[I_{\frac{n(n+1)}{2}} \, \,  0_{\frac{n(n+1)}{2}\times n}\big], \qquad \nonumber
X_i : = I_n \otimes x_i^T, \\
M_i &:=&  X_i D_n, \ \ 
H_i:=  [M_i\, \, I_n], \, \  \nonumber
 G :=  2 \sum_{i=1}^{m} H_i^T H_i,  \ \ 
X:= [x_1^T;x_2^T,\dots;x_m^T;u_1^T;u^T_2;\dots;u_r^T].
\end{eqnarray}
\end{definition}
Consequently, for fixed $i = 1, \dots, m $ and $j = 1, \dots, r $, we obtain the following equations:
\begin{equation*}
     \left\{ \begin{array}{ll}
\frac{1}{2} x_i^TW x_i + x_i^Tb + c=  z^T r_i+ c, & \\
\frac{1}{2}u_j^TW u_j + u_j^Tb + c=  z^T r_j+ c, & \\
Wx_i  =  X_i\ve( W) \, = X_i   D_n \hv(  W) \, = \, M_i \hv(  W) \, = \, M_iw, & \\
W x_i +    b  =  \, M_iw +  I_nb \, = \, H_iz, & \\
\sum_{i=1}^m \|Wx_i+b\|^2 =\sum_{i=1}^m (H_i    z)^T(H_i z) \,= \,    z^T (\sum_{i=1}^{m} (H_i)^T H_i)    z=\frac{1}{2}z^TGz.
\end{array} \right.
\end{equation*}

Note that subindices $i$ and $j$ are consistently used to correspond with data points and Universum points throughout this paper, respectively. 
Using the aforementioned notations, let $z_1$ and $z_2$ be the decision variables constructed using $(W_1, b_1)$ and $(W_2, b_2)$, respectively. These variables are related to the quadratic surfaces for class 1 and class 2, respectively, which will be captured by the twin models discussed later in the paper.
Consequently, we can express the equivalent problem of QTSVM in the standard form of a quadratic minimization problem as follows:

\begin{equation} \tag{QTSVM'}
\label{QTSVM'}
\begin{aligned}
\min_{ z_1, c_1, \xi_1
} \quad & \sum_{i\in \mathcal{I}_1}(z_1^Tr_i+c_1)^2+
C_1 \sum_{i = 1}^{|\mathcal{I}_2|} {\xi_1}_i\\
\textrm{subject to} \quad & 
1+(z_1^Tr_i+c_1)\le{\xi_1}_i, \quad  \forall i\in \mathcal{I}_{2}\\
& z_1\in \R^{\nct+n}, c_1 \in \mathbb R, \xi_1\in \mathbb R^{|\mathcal{I}_{2}|}_+.
\\ \\
\min_{ z_2,c_2, \xi_2
} \quad & \sum_{i\in \mathcal{I}_2}(z_2^Tr_i+c_2)^2+
C_2 \sum_{i = 1}^{|\mathcal{I}_1|} {\xi_2}_i\\
\textrm{subject to} \quad & 
1-(z_2^Tr_i+c_2)\le{\xi_2}_i, \quad  \forall i\in \mathcal{I}_{1}\\
&  
 z_2\in \R^{\nct+n}, c_2 \in \mathbb R, \xi_2\in \mathbb R^{|\mathcal{I}_{1}|}_+.
\end{aligned}
\end{equation}
After obtaining the solutions to these quadratic programs—either by solving them directly or through their dual formulations—the optimal solution to the problem (\ref{QTSVM}) is denoted as $(z_k^*, c_k^*) = (W_k^*, b_k^*, c_k^*)$ for $k = 1, 2$. A new data point $x$ is then assigned to class $k$ based on the following rule:
\begin{equation*}
    \mbox{Class} \, k=\argmin_{i=1,2} \frac{|z_i^Tr(x)+c_i|}{\|H(x)z_i+c_i\|_2^2}=\argmin_{i=1,2} \frac{|\frac{1}{2}x^TW_i x + b_i^Tx+c_i|}{\|W_ix+b_i\|_2^2}.
\end{equation*}

\subsection{Least Squares Quadratic Twin Support Vector Machine}
Combining the ideas of (\ref{ls-tsvm}) and (\ref{QTSVM}), the following model was introduced in \cite{gao2021novel} to capture two quadratic surfaces going through their corresponding classes yet distant from the other class:
\begin{equation} \tag{LS-QTSVM}
\label{LS-QTSVM}
\begin{aligned}
\min_{ W_1, b_1,c_1, \xi_1
} \quad & \sum_{i\in \mathcal{I}_1}(\frac{1}{2}x_i^TW_1x_i+b_1^Tx_i+c_1)^2+
C_1 \sum_{i = 1}^{|\mathcal{I}_2|} {\xi_1}^2_i\\
\textrm{subject to} \qquad & 
1+(\frac{1}{2}{x_i^TW_1 x_i + b_1^Tx_i+c_1)={\xi_1}_i, \quad  \forall i\in \mathcal{I}_{2}}\\
&  W_1\in \mathbb S^n, b_1\in \mathbb R^n, c_1 \in \mathbb R, \xi_1\in \mathbb R^{|\mathcal{I}_{2}|},
\\ \\
\min_{ W_2, b_2,c_2, \xi_2
} \quad & \sum_{i\in \mathcal{I}_2}(\frac{1}{2}x_i^TW_2x_i+b_2^Tx_i+c_2)^2+
C_2 \sum_{i = 1}^{|\mathcal{I}_1|} {\xi_2}^2_i\\
\textrm{subject to} \qquad & 
1-(\frac{1}{2}{x_i^TW_2 x_i + b_2^Tx_i+c_2)={\xi_2}_i, \quad  \forall i\in \mathcal{I}_{1}}\\
&  W_2\in \mathbb S^n, b_2\in \mathbb R^n, c_2 \in \mathbb R, \xi_2\in \mathbb R^{|\mathcal{I}_{1}|}.
\end{aligned}
\end{equation}
This model efficiently separates nonlinear classes, offering both accuracy and speed. Its ability to handle complex relationships without significant computational overhead makes it versatile and suitable for real-world applications across various domains.

The standard quadratic form of the latter model is the following:
\begin{equation} \tag{LS-QTSVM'}
\label{LS-QTSVM'}
\begin{aligned}
\min_{ z_1, c_1, \xi_1
} \quad & \sum_{i\in \mathcal{I}_1}(z_1^Tr_i+c_1)^2+
C_1 \sum_{i = 1}^{|\mathcal{I}_2|} {\xi_1}^2_i\\
\textrm{subject to} \quad & 
1+(z_1^Tr_i+c_1)={\xi_1}_i, \quad  \forall i\in \mathcal{I}_{2}\\
& z_1\in \R^{\nct+n}, c_1 \in \mathbb R, \xi_1\in \mathbb R^{|\mathcal{I}_{2}|},
\\ \\
\min_{ z_2,c_2, \xi_2
} \quad & \sum_{i\in \mathcal{I}_2}(z_2^Tr_i+c_2)^2+
C_2 \sum_{i = 1}^{|\mathcal{I}_1|} {\xi_2}^2_i\\
\textrm{subject to} \quad & 
1-(z_2^Tr_i+c_2)={\xi_2}_i, \quad  \forall i\in \mathcal{I}_{1}\\
&  
 z_2\in \R^{\nct+n}, c_2 \in \mathbb R, \xi_2\in \mathbb R^{|\mathcal{I}_{1}|}_+.
\end{aligned}
\end{equation}
These analogous models have { closed-form} solutions. For $k=1$ and $2$, let us introduce the following matrices:
\begin{equation*} \label{def: R_D_U}
R_k=[r_1, r_2, \dots, r_i, \dots,r_{|\mathcal{I}_{k}|}]
\quad \textrm{and} \quad 
D_k=\diag(y_1,y_2, \dots, y_{|\mathcal{I}_{k}|}).
\end{equation*}
Then, for $k=1$, solving the KKT conditions leads to:
\begin{equation*}
\begin{bmatrix}
R_1R_1^T+C_1R_{2}R_{2}^T
& R_1 e_1 +C_1R_{2}e_{2}
\\
e_1^T R_1^T+C_1 e_{2} R_{2}^T
&
|\mathcal{I}_1|+C_1|\mathcal{I}_{2}|
\end{bmatrix}
\begin{bmatrix}
z_1\\
c_1
\end{bmatrix}=\begin{bmatrix}
C_1R_{2}D_{{2}}e_{2}\\
C_1D_{2}e_{2}
\end{bmatrix}.
\end{equation*}
And, for $k=2$, we get:
\begin{equation*}
\begin{bmatrix}
R_2R_2^T+C_2R_{1}R_{1}^T
& R_2 e_2 +C_2R_{1}e_{1}
\\
e_2^T R_2^T+C_2e_{1} R_{1}^T
&
|\mathcal{I}_2|+C_2|\mathcal{I}_{1}|
\end{bmatrix}
\begin{bmatrix}
z_2\\
c_2
\end{bmatrix}=\begin{bmatrix}
C_2R_{1}D_{{1}}e_{1}\\
C_2D_{1}e_{1}
\end{bmatrix}.
\end{equation*}
After solving these systems for both classes, given a new point $x$, the decision function is based on the same rule analogous to that of QTSVM.

\subsection{$\mathfrak{U}$-QTSVM}
To integrate prior information in the form of Universum points and effectively separate nonlinear datasets, we bring the following model:
\begin{equation} \tag{$\mathfrak{U}$-QTSVM}
\label{UQTSVM}
\begin{aligned}
\min_{ W_1, b_1,c_1, \xi_1
} \ \  & \sum_{i\in \mathcal{I}_1}(\frac{1}{2}x_i^TW_1x_i+b_1^Tx_i+c_1)^2+
C_1 \sum_{i = 1}^{|\mathcal{I}_2|} {\xi_1}_i + {C_u} \sum_{j = 1}^{r} {\psi_{1}}_j \\
\textrm{subject to} \quad & 
1+(\frac{1}{2}{x_i^TW_1 x_i + b_1^Tx_i+c_1)\le{\xi_1}_i, \quad  \forall i\in \mathcal{I}_{2}}\\
    & \frac{1}{2} u_j^T   W_1u_j + b_1^Tu_j + c_1   \ge  -1+\varepsilon-{\psi_{1}}_{j},\quad  j=1,\dots,r,  \nonumber\\
&  W_1\in \mathbb S^n, b_1\in \mathbb R^n, c_1 \in \mathbb R, \xi_1\in \mathbb R^{|\mathcal{I}_{2}|}_+, \psi_{1}\in \mathbb R^{r}_+.
\\ \\
\min_{ W_2, b_2,c_2, \xi_2
} \ \ & \sum_{i\in \mathcal{I}_2}(\frac{1}{2}x_i^TW_2x_i+b_2^Tx_i+c_2)^2+
C_2 \sum_{i = 1}^{|\mathcal{I}_1|} {\xi_2}_i+{C_u}\sum_{j = 1}^{r} {\psi_{2}}_j \\
\textrm{subject to} \quad & 
1-(\frac{1}{2}{x_i^TW_2 x_i + b_2^Tx_i+c_2)\le{\xi_2}_i, \quad  \forall i\in \mathcal{I}_{1}}\\
& -\lp \frac{1}{2} u_j^T   W_2 u_j + b_2^Tu_j + c_2 \rp \ge  -1+\varepsilon-{\psi_{2}}_{j},\quad  j=1,\dots,r,  \nonumber\\
&  W_2\in \mathbb S^n, b_2\in \mathbb R^n, c_2 \in \mathbb R, \xi_2\in \mathbb R^{|\mathcal{I}_{1}|}_+,  \psi_{2}\in \mathbb R^{r}_+.
\end{aligned}
\end{equation}

In the standard quadratic form, the above model becomes:
\begin{equation} \tag{$\mathfrak{U}$-QTSVM'}
\label{UQTSVM'}
\begin{aligned}
\min_{ z_1, c_1, \xi_1,\psi_1
} \quad & \sum_{i\in \mathcal{I}_1}(z_1^Tr_i+c_1)^2+
C_1 \sum_{i = 1}^{|\mathcal{I}_2|} {\xi_1}_i
+C_u\sum_{j=1}^r {\psi_1}_j
\\
\textrm{subject to} \quad & 
1+(z_1^Tr_i+c_1)\le {\xi_1}_i, \quad  \forall i\in \mathcal{I}_{2}\\
& z_1^Tr_j+c_1\ge -1+\varepsilon-{\psi_1}_j \\
& z_1\in \R^{\nct+n}, c_1 \in \mathbb R, \xi_1\in \mathbb R_+^{|\mathcal{I}_{2}|}, \psi_1\in \mathbb R^r.
\\ \\
\min_{ z_2,c_2, \xi_2,\psi_2
} \quad & \sum_{i\in \mathcal{I}_2}(z_2^Tr_i+c_2)^2+
C_2 \sum_{i = 1}^{|\mathcal{I}_1|} {\xi_2}_i+C_u\sum_{j=1}^r {\psi_2}_j\\
\textrm{subject to} \quad & 
1-(z_2^Tr_i+c_2)\le {\xi_2}_i, \quad  \forall i\in \mathcal{I}_{1}\\
& -(z_2^Tr_j+c_2)\ge -1+\varepsilon-{\psi_2}_j \\
&  
 z_2\in \R^{\nct+n}, c_2 \in \mathbb R, \xi_2\in \mathbb R^{|\mathcal{I}_{1}|}_+,\psi_2\in \mathbb R^r.
\end{aligned}
\end{equation}
After solving these quadratic programs directly or through dual theory in convex optimization, a new point is assigned to a class following the same approach outlined for QTSVM.

\subsection{LS-$\mathfrak{U}$-QTSVM}
Incorporating Universum points may escalate the computational complexity of our models. Therefore, opting for a least squares version of the last model is often preferred.
This model can be formulated as follows:
\begin{equation} 
\label{ls-UQTSVM}
\begin{aligned}
\min_{ W_1, b_1,c_1, \xi_1, \psi_1
} \quad & \sum_{i\in \mathcal{I}_1}(\frac{1}{2}x_i^TW_1x_i+b_1^Tx_i+c_1)^2+
C_1 \sum_{i = 1}^{|\mathcal{I}_2|} {\xi_1}^2_i + {C_u}\sum_{j = 1}^{r} {\psi_{1}}^2_j \\
\textrm{subject to} \qquad & 
1+(\frac{1}{2}{x_i^TW_1 x_i + b_1^Tx_i+c_1)={\xi_1}_i, \quad  \forall i\in \mathcal{I}_{2}}\\
    & \frac{1}{2} u_j^T   W_1u_j + b_1^Tu_j + c_1   =  -1+\varepsilon-\psi_{j},\quad  j=1,\dots,r,  \nonumber\\
&  W_1\in \mathbb S^n, b_1\in \mathbb R^n, c_1 \in \mathbb R, \xi_1\in \mathbb R^{|\mathcal{I}_{2}|}, \psi_{1}\in \mathbb R^{r}.
\\ \\
\min_{ W_2, b_2,c_2, \xi_2, \psi_2
} \quad & \sum_{i\in \mathcal{I}_2}(\frac{1}{2}x_i^TW_2x_i+b_2^Tx_i+c_2)^2+
C_2 \sum_{i = 1}^{|\mathcal{I}_1|} {\xi_2}^2_i+{C_u} \sum_{j = 1}^{r} {\psi_{2}}^2_j \\
\textrm{subject to} \qquad & 
1-\big(\frac{1}{2}{x_i^TW_2 x_i + b_2^Tx_i+c_2\big)
  ={\xi_2}_i, \quad  \forall i\in \mathcal{I}_{1}}\\
& -\big( \frac{1}{2} u_j^T   W_2 u_j + b_2^Tu_j + c_2 \big) =-1+\varepsilon-\psi_{j},\quad  j=1,\dots,r,  \nonumber\\
&  W_2\in \mathbb S^n, b_2\in \mathbb R^n, c_2 \in \mathbb R, \xi_2\in \mathbb R^{|\mathcal{I}_{1}|},  \psi_{2}\in \mathbb R^{r}.
\end{aligned}
\tag{LS-$\mathfrak{U}$-QTSVM}
\end{equation}

The above model in the standard quadratic form is as follows:
\begin{equation} \tag{LS-$\mathfrak{U}$-QTSVM'}
\label{LS-UQTSVM'}
\begin{aligned}
\min_{ z_1, c_1, \xi_1,\psi_1
} \quad & \sum_{i\in \mathcal{I}_1}(z_1^Tr_i+c_1)^2+
C_1 \sum_{i = 1}^{|\mathcal{I}_2|} {\xi_1}^2_i
+C_u\sum_{j=1}^r {\psi_1}^2_j
\\
\textrm{subject to} \quad & 
1+(z_1^Tr_i+c_1)= {\xi_1}_i, \quad  \forall i\in \mathcal{I}_{2}\\
& z_1^Tr_j+c_1= -1+\varepsilon-{\psi_1}_j \\
& z_1\in \R^{\nct+n},\, c_1 \in \mathbb R,\, \xi_1\in \mathbb R^{|\mathcal{I}_{2}|},\, \psi_1\in \mathbb R^r.
\\ \\
\min_{ z_2,c_2, \xi_2,\psi_2
} \quad & \sum_{i\in \mathcal{I}_2}(z_2^Tr_i+c_2)^2+
C_2 \sum_{i = 1}^{|\mathcal{I}_1|} {\xi_2}^2_i+C_u\sum_{j=1}^r {\psi_2}^2_j\\
\textrm{subject to} \quad & 
1-(z_2^Tr_i+c_2)= {\xi_2}_i, \quad  \forall i\in \mathcal{I}_{1}\\
& -(z_2^Tr_j+c_2)= -1+\varepsilon-{\psi_2}_j \\
&  
 z_2\in \R^{\nct+n},\, c_2 \in \mathbb R,\, \xi_2\in \mathbb R^{|\mathcal{I}_{1}|},\, \psi_2\in \mathbb R^r.
\end{aligned}
\end{equation}
The closed-form solutions to these problems are obtained as follows.
For $k=1$ and $2$, we define:
\begin{equation*} 
R_{k}=[r_1, r_2, \dots, r_i, \dots,r_{|\mathcal{I}_{k}|}], 
\ \ 
D_{k}=\diag(y_1,y_2,\dots, y_{|\mathcal{I}_{k}|}).
\end{equation*}
Also, let $U=[r_1,r_2,\dots, r_j, \dots, r_r]$. Then, for $k=1$, the solution can be obtained by solving $A_1[z_1;c_1]=b_1$ such that
\begin{equation*}
A_1=\begin{bmatrix}
R_{1}R_{1}^T+C_1R_{2}R_{2}^T+C_uUU^T
& R_{1}e_1 +C_1R_{2}e_{2}+C_uU e_u \\
e_1 R_{1}^T+C_1 e_{2}^T R_{2}^T+C_u e^T_u U^T
&
|\mathcal{I}_{1}|+C_1|\mathcal{I}_{2}|+rC_u
\end{bmatrix}
\end{equation*}
and 
\begin{equation*}
b_1=\begin{bmatrix}
C_1R_{2}D_{{2}} e_{2} +
C_u(-1+\varepsilon)U e_u\\
C_1D_{2} e_{2} +rC_u(-1+\varepsilon)
\end{bmatrix},
\end{equation*}
where $e_u$ is the all-one vector of the dimension of the Universum data. 
And, for $k=2$, the solution is obtained by solving
$A_2[z_2;c_2]=b_2$ where 
\begin{equation*}
A_2=\begin{bmatrix}
R_{2}R_{2}^T+C_2R_{1}R_{1}^T+C_uUU^T
& R_{2}e_2 +C_2R_{1}e_{1}+C_uU e_u \\
e_2 R_{2}^T+C_2 e_{1}^T R_1^T+C_u e^T_u U^T
&
|\mathcal{I}_{2}|+C_2|\mathcal{I}_{1}|+rC_u
\end{bmatrix}
\end{equation*}
and 
\begin{equation*}
b_2=\begin{bmatrix}
C_2R_{1}D_{{1}} e_{1} -C_u(-1+\varepsilon)U e_u\\
C_2D_{1} e_{1} -rC_u(-1+\varepsilon)
\end{bmatrix}.
\end{equation*}
After solving these quadratic programs, either directly or via dual theory in convex optimization, a new point is assigned to a class using the same approach described for QTSVM.

\section{Imbalanced Universum Quadratic Twin Support Vector Machines} \label{sec: imbalance_methodology}
Class imbalance in binary classification poses significant challenges, leading to biased models and compromised predictive performance, particularly for minority classes. This imbalance can result in reduced sensitivity, misinterpretation of accuracy, and skewed decision boundaries, with severe consequences in critical domains like fraud detection or medical diagnoses. To address this, we draw inspiration from \cite{moosaei2023inverse,richhariya2020reduced}, incorporating Universum points to support minority classes in linearly separable datasets. Extending this to quadratic twin support vector machines, we leverage their flexibility to enhance the effectiveness of Universum points. Through this approach, we aim to improve classifier robustness and generalization, especially in imbalanced scenarios.

In the following parts, we maintain a general approach in which the minority class ($A$) is considered the positive class. To address the class imbalance, we employ random undersampling of the negative class ($B$) to create a balanced dataset for forming a quadratic surface for the minority class.
Precisely, letting
$ A=\{x_1,  x_2, \dots, x_{|\mathcal{I}_1|}\}$ and $B=\{\bar x_1,\bar  x_2, \dots, \bar x_{|\mathcal{I}_2|}\}$ with $|\mathcal{I}_1|\ll |\mathcal{I}_2|$, 
we randomly select a reduced sample 
$\tilde B= \{\tilde x_1, \tilde x_2, \dots, \tilde x_{|\mathcal{I}_1|}\}$
from the negative class. We also construct $U=\{u_1, u_2, \dots, u_r\}$ where $r = |\mathcal{I}_2|-|\mathcal{I}_1|$ Universum points using an averaging technique. Next, we choose a reduced sample of them and construct the reduced Universum sample 
$\hat U:=\{\hat u_1, \hat u_2, \dots, \hat u_g\}$
with $g = \lceil {|\mathcal{I}_1|}/{2} \rceil$. Note that $g \ll r$, and these values are deliberately chosen to ensure that the optimization problems associated with each class remain unbiased when an appropriate number of Universum points is added.

The main model of the paper presented next is designed to handle class imbalance when the positive class is assumed to be the minority. In this Imbalanced Universum Quadratic Twin Support Vector Machine (Im-$\mathfrak{U}$-QTSVM) model, we incorporate only $g$ Universum points in the formulation of the minority class to better refine the generalization boundaries in its favor. To address the bias caused by the majority class, we use $r$ Universum points in its formulation, almost as many as the number of points in this class. Slack variables are also used to penalize misclassification in the class points and noise in the Universum points. Furthermore, an $\ell_2$ regularization term on the Hessian of the quadratic surface has been added to enhance the stability and generalization capability of the model.
Consequently, the optimization problem for the Im-$\mathfrak{U}$-QTSVM model, which is based on the hinge loss function, can be formulated as follows:

\begin{equation}
\label{im-u-q-tsvm} \tag{Im-$\mathfrak{U}$-QTSVM}
\begin{aligned}
    \min \quad & \frac{1}{2}\sum_{i\in \mathcal{I}_1} ( \frac{1}{2} { x_i^T  W_1  x_i + b_1^Tx_i + {c_1}} )^2  + \frac{1}{2}C_1 \sum_{i = 1}^{|\mathcal{I}_1|} {\xi_{1}}_i + \frac{1}{2}C_{\hat u}\sum_{j = 1}^{g} {\psi_{1}}_j +\frac{1}{2}\lambda_1 \sum_{i\le j}(W_{1})^2_{ij}\\
   \textrm{s.t.} \quad & -( \frac{1}{2} \tilde{x}_i^T  W_1  \tilde{x}_i + b_1^T\tilde{x}_i + c_1 ) \ge  1 - {\xi_{1}}_i, \quad i=1,\dots,|\mathcal{I}_1| \\
    & (\frac{1}{2} \hat{u}_j^T  W_1\hat{u}_j + b_1^T\hat{u}_j + c_1 )  \ge  -1+\varepsilon-{\psi_1}_j,\quad  j=1,\dots,g,  \nonumber\\
    &  W_1 \in \mathbb{S}^n,\, b_1\in \mathbb{R}^n,\, c_1 \in \mathbb{R},\,
      \xi_1 \in \mathbb{R}^{|\mathcal{I}_1|}_+,\, \psi_1\in \mathbb R^g_+.
\\
 \min \quad & \frac{1}{2}\sum_{i\in \mathcal{I}_2} ( \frac{1}{2} {\bar  x_i^T  W_2 \bar  x_i + b_2^T\bar x_i + {c_2}} )^2  + \frac{1}{2}C_2\sum_{i = 1}^{|\mathcal{I}_1|} {\xi_{2}}_i + \frac{1}{2}C_u \sum_{j = 1}^{r} {\psi_{2}}_j +\frac{1}{2}\lambda_2\sum_{i\le j}(W_{2})^2_{ij}\\
   \textrm{s.t.} \quad &  \frac{1}{2} x_i^T  W_2  x_i + b_2^Tx_i + c_2  \ge  1 - {\xi_{2}}_i, \quad i=1,\dots, |\mathcal{I}_1| \\
    & \frac{1}{2} u_j^T   W_2 u_j+ b_2^Tu_j + c_2   \ge  1-\varepsilon- {\psi_{2}}_{j},\quad  j=1,\dots,r,  \nonumber\\
    &  W_2 \in \mathbb{S}^n,\, b_2 \in \mathbb{R}^n,\, c_2 \in \mathbb{R},\,
       \xi_2 \in \mathbb{R}^{|\mathcal{I}_1|}_+,\,\psi_2\in \mathbb R^r_+.
\end{aligned}
\end{equation}

Similarly, the least squares version of the model discussed above can manage class imbalance when the positive class is the minority. In the Imbalanced Least Squares Universum Quadratic Twin Support Vector Machine (Im-LS-$\mathfrak{U}$-QTSVM) model, we also utilize $g$ Universum points in the minority class formulation to better refine the generalization boundaries in its favor. Similarly, to mitigate the bias caused by the majority class, we employ $r$ Universum points in its formulation, which is nearly equivalent to the number of points in this class. Slack variables are incorporated to penalize misclassification in the class points and noise in the Universum points as well. Additionally, an $\ell_2$ regularization term is added to the Hessian of the quadratic surface to improve the stability and generalization capability of the model. Further, we consider the quadratic function as the loss function. 
Therefore, the optimization problem for the Im-LS-$\mathfrak{U}$-QTSVM model can be expressed as follows:

\begin{equation}
\label{im-ls-u-q-tsvm} \tag{Im-LS-$\mathfrak{U}$-QTSVM}
\begin{aligned}
    \min \quad & \frac{1}{2}\sum_{i\in \mathcal{I}_1} ( \frac{1}{2} { x_i^T  W_1  x_i + b_1^Tx_i + {c_1}} )^2  + \frac{1}{2}C_1 \sum_{i = 1}^{|\mathcal{I}_1|} {\xi_{1}}^2_i+ \frac{1}{2}C_{\hat u}\sum_{j = 1}^{g} {\psi_{1}}^2_j +\frac{1}{2}\lambda_1 \sum_{i\le j}(W_{1})^2_{ij}\\
   \textrm{s.t.} \quad & 1+ \frac{1}{2} \tilde{x}_i^T  W_1  \tilde{x}_i + b_1^T\tilde{x}_i + c_1 = {\xi_{1}}_i, \quad i=1,\dots,|\mathcal{I}_1| \\
    & -1+\varepsilon-(\frac{1}{2} \hat{u}_j^T   W_1\hat{u}_j + b_1^T\hat{u}_j+ c_1 ) =  {\psi_{1}}_j,\quad  j=1,\dots,g,  \nonumber\\
    &  W_1 \in \mathbb{S}^n,  b_1\in \mathbb{R}^n, c_1 \in \mathbb{R},  \xi_1 \in \mathbb{R}^{|\mathcal{I}_1|}, \psi_1\in \mathbb R^g.
\\
 \min \quad & \frac{1}{2}\sum_{i\in \mathcal{I}_2} ( \frac{1}{2} { \bar x_i^T  W_2 \bar   x_i + b_2^T\bar x_i + {c_2}} )^2  + \frac{1}{2}C_2\sum_{i = 1}^{|\mathcal{I}_1|} {\xi_{2}}^2_i+ \frac{1}{2}C_u \sum_{j = 1}^{r} {\psi_{2}}^2_j +\frac{1}{2}\lambda_2\sum_{i\le j}(W_{2})^2_{ij}\\
   \textrm{s.t.} \quad &  1-(\frac{1}{2} x_i^T  W_2  x_i + b_2^Tx_i + c_2) =  {\xi_{2}}_i, \quad i=1,\dots,|\mathcal{I}_1| \\
    &  1-\varepsilon-(\frac{1}{2} u_j^T   W_2 u_j+ b_2^Tu_j + c_2)   = {\psi_{2}}_{j},\quad  j=1,\dots,r,  \nonumber\\
    &  W_2 \in \mathbb{S}^n,  b_2 \in \mathbb{R}^n, c_2 \in \mathbb{R},  \xi_2 \in \mathbb{R}^{|\mathcal{I}_1|},\psi_2\in \mathbb R^r.
\end{aligned}
\end{equation}

The constraints of the above least squares version model are equations, so by substituting $\xi$ and $\psi$ into the objective function, and via Definition \ref{def: definitios} and letting
\begin{equation*}
    \begin{aligned}
S_A & 
=[r_1,r_2,\dots, r_{|\mathcal{I}_1|}],  
\\
S_B & 
=[\bar r_1, \bar r_2, \dots, \bar r_{|\mathcal{I}_2|}],
\\
S_{\tilde B} & =
[\tilde r_1, \tilde r_2, \dots, \tilde r_{|\mathcal{I}_2|}],
\\
S_{U}&=[r_1, r_2, \dots, r_j,  \dots,  r_{r}],
\\
S_{\hat U}&=[\hat r_1,\hat r_2,  \dots, \hat r_j, \dots, \hat r_{g}],
    \end{aligned}
\end{equation*}    
the problem for the (first or) minority class can be reformulated as 
\begin{equation}
    \label{lsqReform2}  \tag{Im-LS-$\mathfrak{U}$-QTSVM-1}
\begin{aligned}
   \min_{z_1,c_1}\
 & f=\dfrac{1}{2}\|S_A^Tz_1+c_1e\|_2^2
  +\frac{C_1}{2}\big \| e +{S^{T}_{\tilde B}}z_1+c_1e\big\|^{2}_2
  \notag \\
  &\qquad +\frac{C_{\hat u}}{2} \big \| (-1+\varepsilon) e -({S^{T}_{\hat U}}z_1+c_1e)\big \|^{2}_2+ \frac{\lambda_1}{2} \|Vz_1 \|^{2}_{2}.
\end{aligned}
\end{equation}

Here, $e$ is a vector of ones of appropriate size and $V$ was introduced in Def.~\ref{def: definitios}. 
This problem can be solved by putting the gradient with respect to $z$ and $c$ equal to zero, so we have the following equations: 
\begin{align}\label{eqImLsUQTSVMpartialZ}
\dfrac{\partial f}{\partial z}&=S_A(S_A^Tz_1+c_1e)
+C_1 {S_{\tilde B}}(e+{S^{T}_{\tilde B}}z_1+c_1e)
\\
&~~~
-C_{\hat u} {S^{T}_{\hat U}}((-1+\varepsilon) e-
({S^{T}_{\hat U}}z_1+c_1e))+\lambda_1 V^{T}Vz_1=0,\\
\dfrac{\partial f}{\partial c}&=
e^T(S_A^Tz_1+c_1e)+
C_1e^{T}(e+{S^{T}_{\tilde B}}z_1+c_1e)
-C_{\hat u} e^{T}((-1+\varepsilon) e-({S^{T}_{\hat U}}z_1+c_1e))=0.
\end{align}
By integrating the above equations, we have the following system:
\begin{align}
&
\left[ 
\begin{matrix}
S_AS_A^T  + C_1 {S_{\tilde B}}{S^{T}_{\tilde B}} +C_{\hat u}  {S_{\hat U}}{S^{T}_{\hat U}}+ \lambda_1 V^{T}V & Ae+C_1 {S_{\tilde B}}e + C_{\hat u}  {S_{\hat U}}e \\
e^T{S^T_A}+C_1 e^{T}{S^{T}_{\tilde B}} + C_{\hat u}  e^{T}{S^{T}_{\hat U}} & C_1 e^{T}e + C_{\hat u}  e^{T}e+|\mathcal{I}_1|
\end{matrix}\right] \left[ \begin{matrix}
z_1\\
c_1
\end{matrix}\right] \nonumber 
\vspace{0.9cm}
\\ 
&  = \left[ \begin{matrix}
-C_1 {S_{\tilde B}}e+(-1+\varepsilon)C_{\hat u} {S_{\hat U}} e\\
-C_1 e^{T}e+(-1+\varepsilon)C_{\hat u} e^{T} e
\end{matrix}\right]. \nonumber
\end{align}
Let 
$$\Sigma = \left[ \begin{matrix}
S_AS_A^T + C_1 {S_{\tilde B}}{S^{T} _{\tilde B}}+C_{\hat u} {S_{\hat U}}{S^{T}_{\hat U}}+ \lambda_1 V^{T}V & Ae+ C_1 {S_{\tilde B}}e + C_{\hat u} {S_{\hat U}}e \\
e^TA^T+C_1 e^{T}{S^{T}_{\tilde B}} + C_{\hat u}  e^{T}{S^{T}_{\hat U}} & C_1 e^{T}e + C_{\hat u} e^{T}e+|\mathcal{I}_1|
\end{matrix}\right],
$$
and assume it is invertible. Let 
$$\beta =\left[ \begin{matrix}
-C_1 {S_{\tilde B}}e+(-1+\varepsilon) C_{\hat u}{S_{\hat u}}e\\
-C_1 e^{T}e+(-1+\varepsilon) C_{\hat u} e^{T}e
\end{matrix}\right],
$$
then 
\begin{align*}
\left[ \begin{matrix}
z_1\\
c_1
\end{matrix}\right]= \Sigma^{-1}\beta.
\end{align*}
Similarly, the problem corresponding to the majority class in the least squares loss function becomes:
\begin{equation}
    \label{lsqReform22} \tag{Im-LS-$\mathfrak{U}$-QTSVM-2}
\begin{aligned}
   \min_{z_2,c_2}\
  &\dfrac{1}{2}\|S^T_Bz_2+c_2e\|_2^2+\frac{C_2}{2} \big \| e -(S_A^{T}z_2+c_2e)\big\|^{2}_2 \notag\\
  &
  +\frac{C_u}{2} \big \| (1-\varepsilon) e -(S_U^{T}z_2+c_2e)\big \|^{2}_2+\frac{\lambda_2}{2} \|Vz_2 \|^{2}_{2}, 
\end{aligned}
\end{equation}
which, by the same methodology, leads to
$$\tilde\Sigma = \left[ \begin{matrix}
S_BS^T_B + C_2 S_AS^T_A +C_u  {S_US^T_U}+ \lambda_2 V^{T}V & S_Be+C_2 S_Ae + C_u  S_Ue \\
e^TS_B^T+C_2 e^{T}S_A^{T} + C_u  e^{T}S_U^{T} & C_2 e^{T}e + C_u  e^{T}e+|\mathcal{I}_2|
\end{matrix}\right],$$
and assume it is invertible. Let 
$$\tilde\beta =\left[ \begin{matrix}
C_2S_BS_Ae+(1-\varepsilon) C_u S_BS_Ue\\
C_2 e^{T}e+(1-\varepsilon)C_u e^{T} e
\end{matrix}\right],$$
then 
\begin{align*}
\left[ \begin{matrix}
z_2\\
c_2
\end{matrix}\right]= \tilde\Sigma^{-1}\tilde\beta.
\end{align*}

A new data point $x \in \mathbb{R}^{n}$ is assigned to class $i \in \{+1, -1\}$ using a rule similar to that of the QTSVM.

{ \begin{remark}  For the least squares variant, the closed-form solution guarantees global optimality. For the QP variant, standard convex optimization theory ensures convergence to a global optimum given the convexity of the objective.
\end{remark}}
{ 
\begin{algorithm}[H]
\caption{  Imbalanced Universum Quadratic Twin SVM (Im-$\mathfrak{U}$-QTSVM)}
\textbf{Input:} 
\begin{itemize}
    \item Training sets $A$ (minority class), $B$ (majority class)  
    \item Regularization and penalty parameters $C_1, C_2, C_u, C_{\hat u}, \lambda_1, \lambda_2$  
    \item Universum parameter $\varepsilon$
\end{itemize}
\textbf{Output:} Quadratic surfaces $(W_1, b_1, c_1)$ for the minority class and $(W_2, b_2, c_2)$ for the majority class

\begin{enumerate}
    \item Randomly undersample the majority class $B$ to obtain a balanced subset $\tilde B$ of size $|\mathcal{I}_1|$.
    
    \item Construct Universum points:
    \begin{itemize}
        \item $U = \{u_1, \dots, u_r\}$, where $r = |\mathcal{I}_2| - |\mathcal{I}_1|$, representing Universum points for the majority class.  
        \item $\hat U = \{\hat u_1, \dots, \hat u_g\}$, where $g = \lceil |\mathcal{I}_1|/2 \rceil$, representing Universum points for the minority class.
    \end{itemize}
         
\item Solve the least-squares optimization problems \eqref{im-ls-u-q-tsvm} by solving the corresponding systems of linear equations for each class, and obtain the classifier parameters:
$$(W_1, b_1, c_1) \quad \text{for the minority class, and} \quad (W_2, b_2, c_2) \quad \text{for the majority class}.$$
    
    \item Classify a new sample $x \in \mathbb{R}^n$ by computing the distances to both quadratic surfaces:

    \begin{equation*}
        \text{Class } k = \arg\min_{i=1,2} \frac{\big|\frac{1}{2} x^T W_i x + b_i^T x + c_i \big|}{\|W_i x + b_i\|_2^2}.
    \end{equation*}

\end{enumerate}
\end{algorithm}
}

\section{Theoretical Properties of Proposed Models} \label{sec: theoretical_properties}

In this section, we present some theoretical properties of the proposed model. We first show that optimality is always ensured. Moreover, under very mild assumptions, the optimal solution is unique.

\begin{theorem}\label{thmOptSolEx}
The problem \eqref{im-ls-u-q-tsvm} always has an optimal solution.
\end{theorem}

\begin{proof}
The problem minimizes a convex quadratic function on an affine subspace, so it possesses an optimum.    
\end{proof}

\begin{theorem}
Problem \eqref{im-ls-u-q-tsvm} has a unique optimal solution if and only if the set of vectors 
$\{x_i,\,i\in \mathcal{I}_1;\;\tilde{x}_i,\,i=1,\dots, |\mathcal{I}_1| ;\;
\hat{u}_j,\, j=1,\dots,g\}$ 
is affinely independent.
\end{theorem}

\begin{proof}
From Theorem~\ref{thmOptSolEx}, an optimum always exists. Recall that the problem has the form of minimization of a convex quadratic function on an affine subspace. Thus, the minimum is not unique if and only if the objective function is constant in the unbounded direction. This happens if and only if there is a nontrivial solution of  
\begin{align*}
\frac{1}{2} \tilde{x}_i^T  W_1  \tilde{x}_i + b_1^T\tilde{x}_i + c_1 
 &= {\xi_{1}}_i, \quad i=1,\dots, |\mathcal{I}_1| \\
-(\frac{1}{2} \hat{u}_j^T   W_1\hat{u}_j 
 + b_1^T\hat{u}_j+ c_1 ) 
 &=  {\psi_{1}}_j,\quad  j=1,\dots,g,
\end{align*}
for which the objective function vanishes. If the objective function is zero, then we can deduce that 
$\xi_{1}=0$, $\psi_{1}=0$ and $W_1=0$. Hence, the remaining conditions take the form of
\begin{align*}
b_1^T\tilde{x}_i + c_1  &= 0, \quad i=1,\dots, |\mathcal{I}_1| \\
b_1^T\hat{u}_j+ c_1 &= 0,\quad  j=1,\dots,g,\\
b_1^T{x}_i + c_1  &= 0, \quad \forall i\in \mathcal{I}_1 .
\end{align*}
This characterizes the linear dependence of vectors 
$$
\{(x_i;1),\,i\in \mathcal{I}_1;\;
 (\tilde{x}_i;1),\,i=1,\dots, |\mathcal{I}_1| ;\;
 (\hat{u}_j;1),\, j=1,\dots,g\},
$$ 
or, equivalently, the affine dependence of the vectors in the formulation of the theorem.
\end{proof}

Naturally, the situation in which $z^*=0$ for an optimal solution $(z^*,c^*)$ is not desirable. Thus, we present a sufficient condition that ensures the optimal solution is nonzero. This condition assumes that three specific vectors in space $\R^{\frac{n(n+1)}{2}}$ are linearly independent, which is highly likely to be satisfied.

\begin{proposition}
Let $(z^*,c^*)$ be an optimal solution of \eqref{lsqReform2}. 
If vectors $S_Ae$, $S_{\tilde B}e$, and $S_{\hat U}e$ are linearly independent,  
then $z^*\not=0$.
\end{proposition}

\begin{proof}
Suppose, on the contrary, $z^*=0$. From equations \eqref{eqImLsUQTSVMpartialZ} we obtain that
\begin{align*}
c_1S_Ae + C_1(1+c_1) {S_{\tilde B}}e
-C_{\hat u}(-1+\varepsilon-c_1) {S^{T}_{\hat U}}e=0.    
\end{align*}
This means that vectors $S_Ae$, $S_{\tilde B}e$, and $S_{\hat U}e$ are linearly dependent; a contradiction.
\end{proof}

{ \begin{remark}[Computational Complexity]
The model (Im-$\mathfrak{U}$-QTSVM) is based on solving two convex quadratic optimization problems. For the  class $i\in\{1,2\}$, it uses $\mathcal{O}(n^2)$ variables $W_i, b_i, c_i$  and $\mathcal{O}(|\mathcal{I}_i|)$ slack variables $\xi_i,\psi_i$. Employing a suitable interior-point method, the resulting complexity is $\mathcal{O}((n^2 + |\mathcal{I}_i|)^3)$ for class~$i$. 

For the model (Im-LS-$\mathfrak{U}$-QTSVM), the least squares formulation allows a closed-form solution via two linear systems of equations $\Sigma z = \beta$ and $\tilde{\Sigma} z = \tilde{\beta}$. The constraint matrices require some calculations. The matrix $S_A$ has a size of order $n^2\times |\mathcal{I}_1|$, whence the matrix $S_AS_A^T$ is constructed in time $\mathcal{O}(n^4|\mathcal{I}_1|)$. Similarly, the construction of matrices $S_{\tilde{B}}S_{\tilde{B}}^T$ and $S_{\hat{U}}S_{\hat{U}}^T$ needs $\mathcal{O}(n^4|\mathcal{I}_1|)$ and $\mathcal{O}(n^4g)$, respectively. Since the value of $g$ is of order $\mathcal{O}(|\mathcal{I}_1|)$, the construction of matrix $\Sigma$ requires a total time of $\mathcal{O}(n^4|\mathcal{I}_1|)$. The matrix $\Sigma$ has a size of $n^2\times n^2$, so solving the system of linear equations by the standard methods requires time $\mathcal{O}(n^6)$. Therefore, the total computational time for the first system of equations is $\mathcal{O}(n^4(n^2+|\mathcal{I}_1|))$. For the second system of equations, we obtain the time complexity $\mathcal{O}(n^4(n^2+|\mathcal{I}_2|))$ analogously. 

Hence, the least squares variant is more efficient for large datasets, in particular when $|\mathcal{I}_i|\gg n^2$.
\end{remark}}

\section{Numerical Experiments}
\label{Sec: Numerical_Experiments}
{\color{black}
In this section, a large scale of numerical experiments are conducted to validate the performance of the proposed Im-LS-$\mathfrak{U}$-QTSVM model for binary classification. We first introduce the experimental settings, including the data sources, the parameter setups, etc. Then, all the numerical experiments are conducted to validate the classification accuracy of the proposed model, along with some related binary classification models on public benchmark datasets. Certain statistical tests are also conducted to analyze the computational results.
}

\subsection{Experiment Settings}
\label{subsection: effect of parameters and Universum data}

{\color{black}

In the numerical experiments, multiple benchmark SVM models are implemented for comparison, including the linear SVM, linear Twin SVM, and least squares twin SVM models. We also tested their corresponding Universum variants. All the tested models and their abbreviations are listed in Table \ref{table: models & abbreviations for experiments}. In addition, we list the Python toolboxes (commercial solvers, packages, etc.) in Table \ref{table: models & abbreviations for experiments}. All experiments are conducted on a MacBook Pro with an Apple M3 Pro chip and 18 GB of unified memory.

\begin{table}[h]
\centering
\caption{\textcolor{black}{Models implemented for the numerical experiments}}
\label{table: models & abbreviations for experiments}
\begin{tabular}{l l l l} 
\hline
Model & Abbreviation &  Toolbox & Parameters\\ 
\hline
SVM & SVM &  Scikit-learn &  $C$ \\
Cost-sensitive SVM & CSSVM & Scikit-learn & $C$ \\
Twin SVM & TSVM & CVXOPT &  $C_{1}, C_{2}$ \\
Least Squares Twin SVM & LS-TSVM & Numpy &  $C_{1}, C_{2}$ \\
Quadratic Least Squares Twin SVM & LS-QTSVM & 
Numpy &  $C_{1}, C_{2}$ \\
Universum Twin SVM & $\mathfrak{U}$-TSVM & 
CVXOPT &  $C_{1}, C_{2}, C_u, \varepsilon$ \\
Least Squares Universum Twin SVM & LS-$\mathfrak{U}$-TSVM & 
Numpy &  $C_{1}, C_{2}, C_u, \varepsilon$ \\
Universum Quadratic Twin SVM & $\mathfrak{U}$-QTSVM & 
CVXOPT &  $C_{1}, C_{2}, C_u, \varepsilon$ \\
\makecell[l]{Least Squares Universum Quadratic \\ Twin SVM} & LS-$\mathfrak{U}$-QTSVM & 
Numpy &  $C_{1}, C_{2}, C_u, \varepsilon$ \\
\makecell[l]{Universum Weighted Kernel \\ Extreme Learning Machine} & URKWELM & Numpy & $C, C_u, \gamma$ \\
\makecell[l]{Intuitionistic Fuzzy Twin \\ Proximal SVM} & FHTPSVM & Numpy &  \makecell[l]{$M_1, M_2, M_3, M_4,$ \\ $\gamma_1, \gamma_2, I_w$} \\
\makecell[l]{Imbalanced Least Squares Universum \\ Quadratic Twin SVM} & Im-LS-$\mathfrak{U}$-QTSVM & 
Numpy &  $C_{1}, C_{2}, C_u, \lambda, \varepsilon$ \\
\hline
\end{tabular}
\end{table}

For each experiment, a five-fold cross-validation procedure is applied. Each experiment is repeated ten times for each tested model to keep the results statistically meaningful. Means and standard deviations of the accuracy scores are recorded as the measurement of classification effectiveness. 
{ \paragraph{Universum Data Generation.} 
To construct Universum data, we adopt the averaging technique described in~\cite{moosaei2023universum}. 
The procedure is detailed as follows:

\begin{enumerate}
    \item Randomly select $10\%$ of samples from each class $A$ and $B$.
    \item Pair the selected samples between the two classes.
    \item For each pair $(x_a, x_b)$, compute a Universum point
    \[
        u = \tfrac{1}{2}(x_a + x_b).
    \]
    \item Collect all such $u$ as the Universum dataset $U$.
\end{enumerate}}

All the parameters for each model are tuned by applying the grid-search, which is commonly adopted for tuning SVM parameters \cite{gao2021novel,luo2016soft}. For each pair of the twin models in \ref{tsvm}, \ref{ls-tsvm}, \ref{LS-QTSVM},  \ref{utsvm}, \ref{ls-u-tsvm},  \ref{UQTSVM}, and \ref{ls-UQTSVM}, the penalty parameters $C_{1}$ and $C_{2}$ are selected to be the identical. {\color{black} In model FHTPSVM, we set $M_1 = M_2$ and $M_3 = M_4$, and $\gamma_1 = \gamma_2 = I_w = 0.5$. In model FHTPSVM, we set kernel parameter $\gamma$ to be the reciprocal of the feature amount. The ranges of the rest of  tuning parameters $C$, $C_u$, $C_1$, $C_2$, $M_1$, $M_3$, $\lambda$ and $\varepsilon$ are searched in grid $\{2^{-8}, 2^{-7}, \dots, 2^{8}\}$.

}

We first generate four artificial datasets to show the flexibility of the proposed Im-LS-$\mathfrak{U}$-QTSVM model. The surfaces produced by the proposed model are visualized in Figure \ref{figure: arti-data results Im-LS-U-QTSVM}, in which Figures \ref{figure: arti-1-rate3} and  \ref{figure: arti-1-rate10} have the similar nonlinear pattern while Figures \ref{figure: arti-2-rate3} and \ref{figure: arti-2-rate10} have the other similar one. To show the proposed model on different levels of imbalanced data, we set the imbalanced rates in Figure \ref{figure: arti-1-rate3} and \ref{figure: arti-2-rate3} to be three while the imbalanced rates in Figure \ref{figure: arti-1-rate10} and \ref{figure: arti-2-rate10} to be ten. The major class of data is in color green, while the minor class is in red. For the major class, the number of solid points is the same as the minor class.

\begin{figure}[h]
    \centering
    \begin{subfigure}[b]{0.46\textwidth}
        \includegraphics[width=\textwidth]{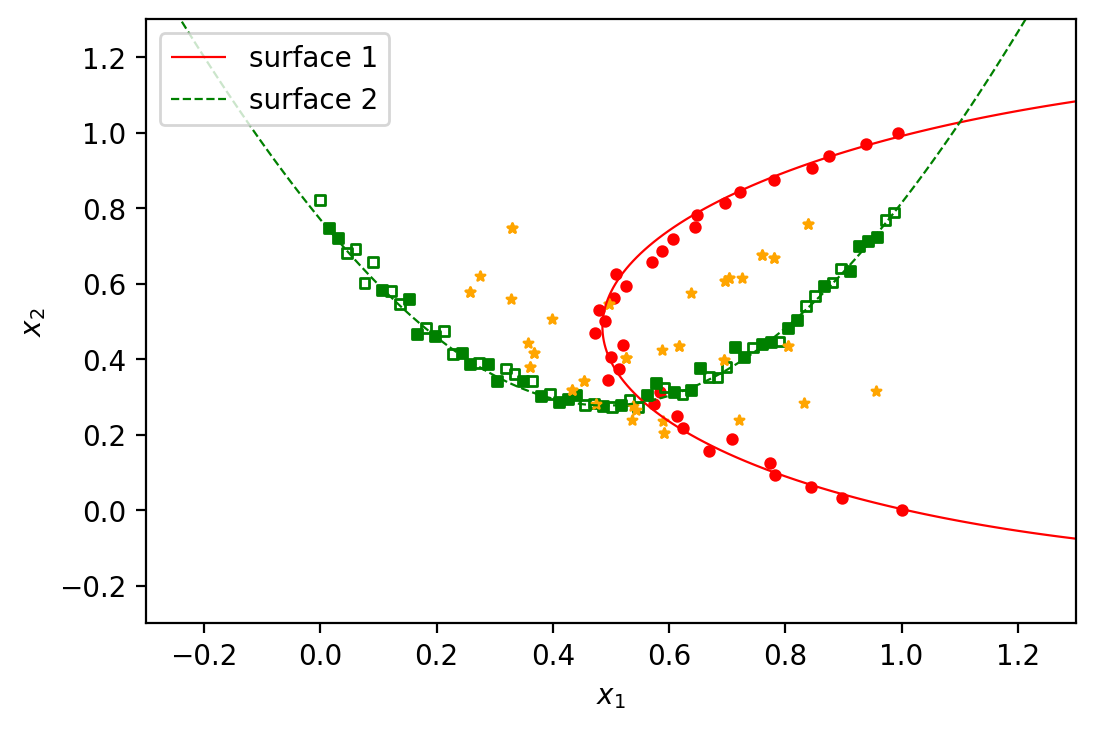}
        \caption{Arti-1, rate 1:3}
        \label{figure: arti-1-rate3}
    \end{subfigure}
    ~ 
    \begin{subfigure}[b]{0.46\textwidth}
        \includegraphics[width=\textwidth]{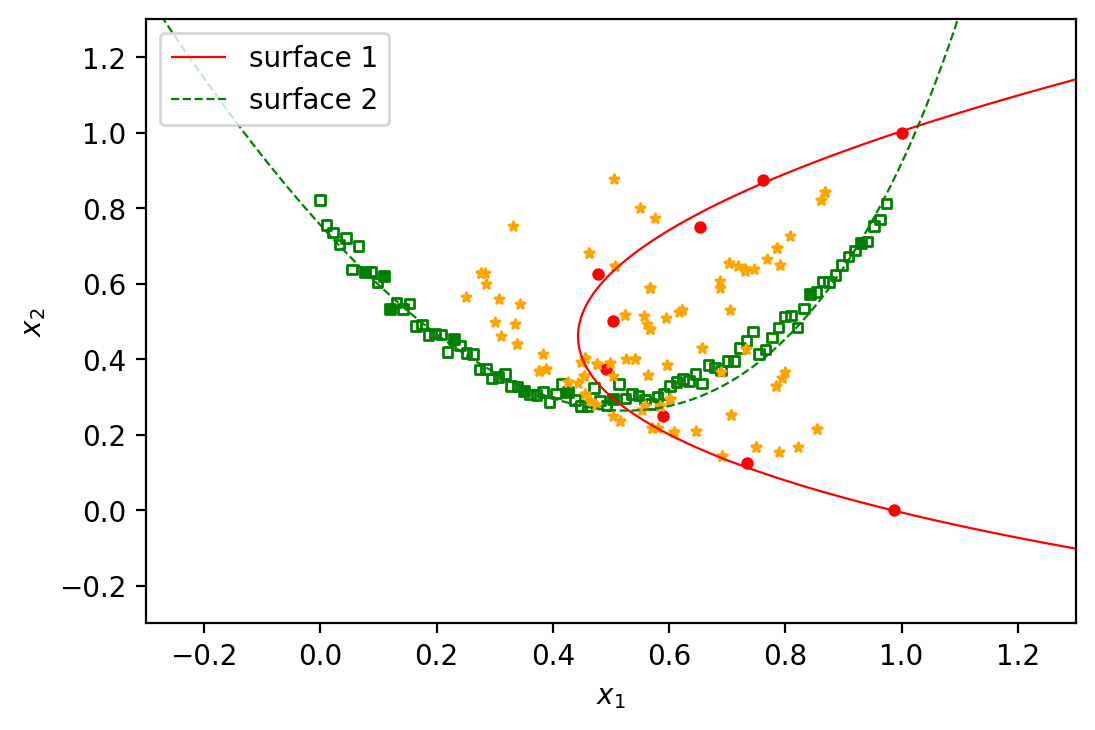}
        \caption{Arti-1, rate 1:10}
        \label{figure: arti-1-rate10}
    \end{subfigure}

    \begin{subfigure}[b]{0.46\textwidth}
        \includegraphics[width=\textwidth]{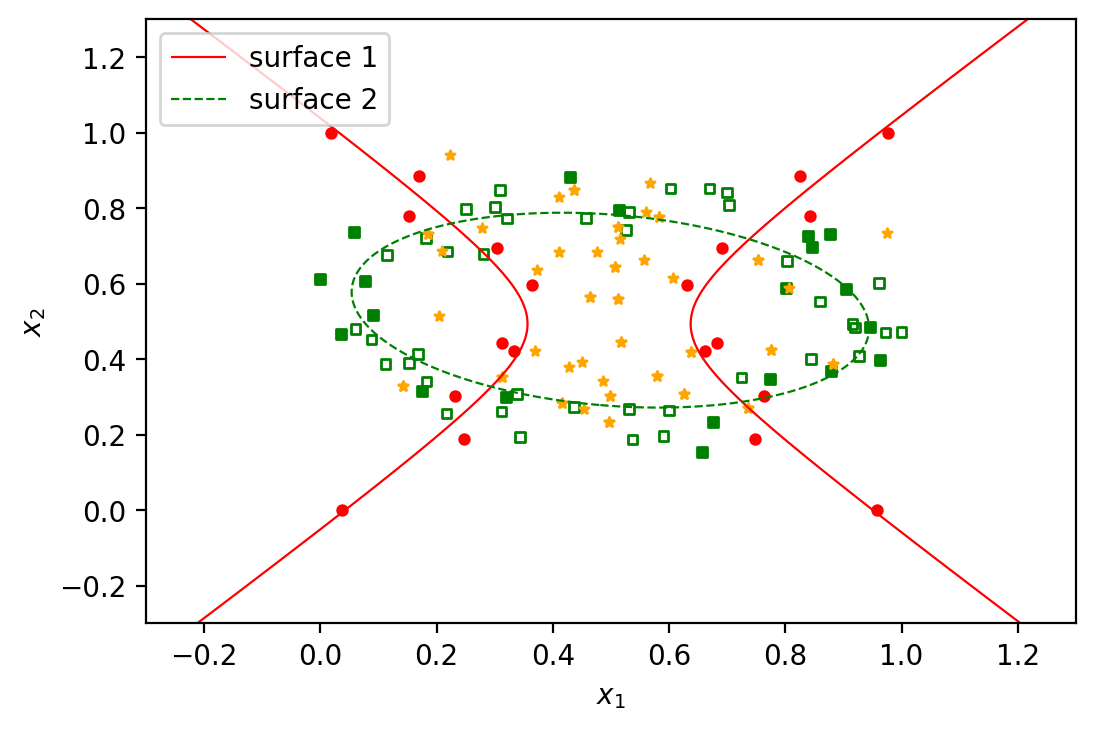}
        \caption{Arti-2, rate 1:3}
        \label{figure: arti-2-rate3}
    \end{subfigure}
    ~ 
    \begin{subfigure}[b]{0.46\textwidth}
        \includegraphics[width=\textwidth]{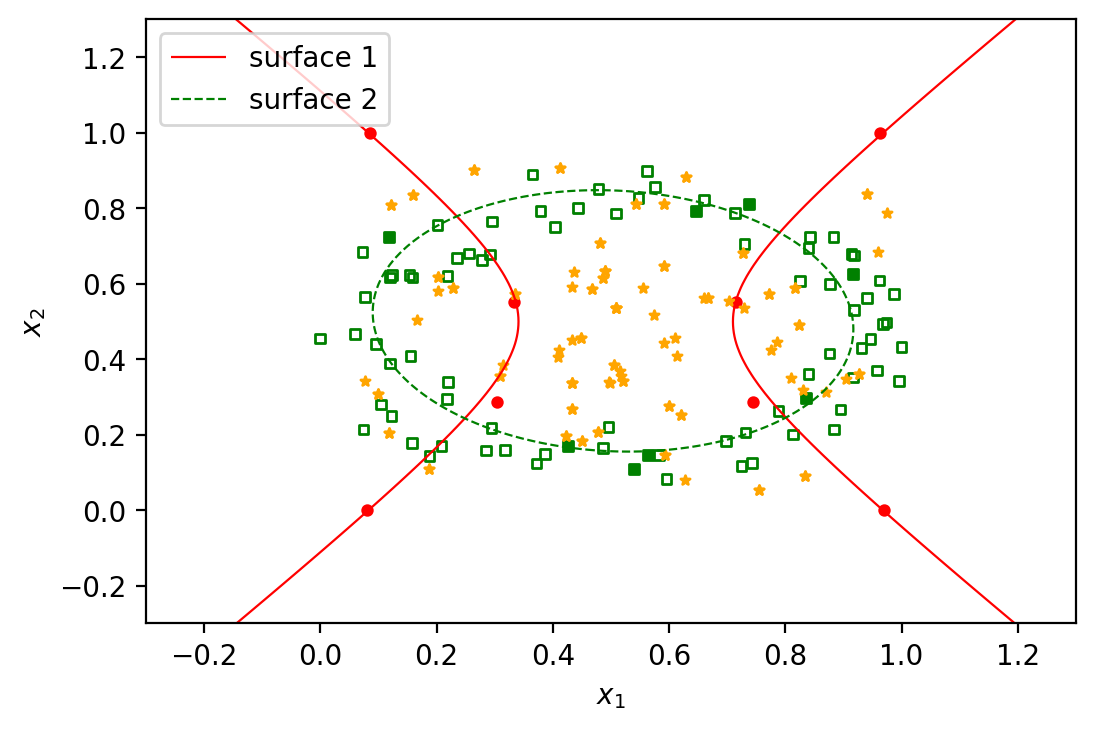}
        \caption{Arti-2, rate 1:10}
        \label{figure: arti-2-rate10}
    \end{subfigure}

    \caption{Im-LS-U-QTSVM on artificial data sets}\label{figure: arti-data results Im-LS-U-QTSVM}
\end{figure}

From the results in Figure \ref{figure: arti-data results Im-LS-U-QTSVM} we notice that the proposed Im-LS-$\mathfrak{U}$-QTSVM model can fit the quadratic patterns. It also provides promising classification accuracy when the dataset is imbalanced under different levels. 
\subsection{Results on Public Benchmark Data}
\label{subsection: numerical section public benchmark datasets}

The proposed Im-LS-$\mathfrak{U}$-QTSVM model is tested on some public benchmark data sets. All the data information is listed in Table \ref{table: pubic benchmark data info}

\begin{table}[t]
  \centering
    \begin{tabular}{l l l l}
  \hline Data set & \# of data points (Class 1/Class 2) & \# of features & Imbalance rate \\
  \hline
pima  &  768  (268/500) &  8  &  1.87  \\
blood transfusion  &  748  (178/570) &  4  &  3.20  \\
seismic bumps  &  2584  (170/2414) &  15  &  14.20  \\
haberman  &  306  (81/225) &  3  &  2.78  \\
wholesale  &  440  (142/298) &  7  &  2.10  \\
liver disorder  &  319  (134/185) &  6  &  1.38  \\
wine  &  130  (59/71) &  13  &  1.20  \\
balance  &  576  (288/288) &  4  &  1.00  \\
taxpayer  &  667  (331/336) &  9  &  1.02  \\
car evaluation  &  1594  (384/1210) &  6  &  3.15  \\
CTG  &  2126  (471/1655) &  22  &  3.51  \\
US heart  &  3658  (557/3101) &  15  &  5.57  \\
Alzheimer rowdel  &  354  (37/317) &  9  &  8.57  \\
Alzheimer clinical  &  1200  (373/827) &  5  &  2.22  \\
thyroid  &  215  (65/150) &  5  &  2.31  \\
ecoli 1  &  336  (143/193) &  7  &  1.35  \\
ecoli 123  &  336  (20/316) &  7  &  4.25  \\
ecoli 4  &  336  (64/272) &  7  &  8.60  \\
ecoli 1234  &  336  (9/327) &  7  &  10.59  \\
ecoli 5  &  336  (29/307) &  7  &  15.80  \\
ecoli 678  &  336  (35/301) &  7  &  36.33  \\
    \hline
    \end{tabular}
    
  \caption{Basic information of public benchmark data sets.\protect\footnote{The sources of data sets can be found here: \url{https://github.com/tonygaobasketball/Sparse-UQSSVM-Models-for-Binary-Classification}}}
  \label{table: pubic benchmark data info}
\end{table}  


\begin{sidewaystable} 
\centering
\resizebox{\textwidth}{!}{
\begin{tabular}{c|ccccccccBBBc}\hline

\multirow{2}{*}{Dataset} & \multicolumn{12}{c}{Accuracy score (mean $\pm$ standard deviation) (\%)}  \\
\cline{2-13}
  &  LS-$\mathfrak{U}$-QTSVM      &  LS-$\mathfrak{U}$-TSVM      &  LS-QTSVM      &  LS-TSVM      &  $\mathfrak{U}$-QTSVM      &  $\mathfrak{U}$-TSVM      &  TSVM      &  SVM   & CSSVM & URKWELM & FHTPSVM &  Im-LS-$\mathfrak{U}$-QTSVM      \\
\hline

pima  & $  77.74  \pm  2.08  $ & $  77.48  \pm  1.97  $ & $  77.48  \pm  1.18  $ & $  67.45  \pm  0.93  $ & $  77.48  \pm  1.67  $ & $  78.00  \pm  2.10  $ & $  70.19  \pm  4.18  $ & $  77.87  \pm  1.63 $ & $ 76.04 \pm 3.58 $ & $ 72.59 \pm 2.86 $ & $ 77.15 \pm 2.01  $ & $   \textbf{78.27}  \pm  3.66  $ \\
blood transfusion  & $  79.61  \pm  1.61  $ & $  79.48  \pm  1.09  $ & $  79.21  \pm  1.91  $ & $  76.61  \pm  0.88  $ & $  79.14  \pm  1.39  $ & $  78.34  \pm  1.52  $ & $  76.27  \pm  0.27  $ & $  76.27  \pm  0.27 $ &  $  75.27  \pm  6.33
$ &  $  77.00  \pm  1.17
$ &  $  78.47  \pm  1.73   $ & $  \textbf{80.42}  \pm  3.14  $ \\
seismic bumps  & $  93.32  \pm  0.09  $ & $  93.30  \pm  0.18  $ & $  92.86  \pm  0.34  $ & $  93.40  \pm  0.16  $ & $  93.11  \pm  0.31  $ & $  \textbf{93.42}  \pm  0.06  $ & $  90.71  \pm  5.19  $ & $  \textbf{93.42}  \pm  0.01  $ &  $  80.03  \pm  1.67
$ &  $  93.44  \pm  0.06
$ &  $  93.42  \pm  0.08  $ & $  \textbf{93.42}  \pm  0.01  $ \\
haberman  & $  76.96  \pm  3.31  $ & $  76.97  \pm  2.56  $ & $  75.49  \pm  2.48  $ & $  74.02  \pm  0.77  $ & $  75.49  \pm  1.49  $ & $  76.48  \pm  4.05  $ & $  73.69  \pm  0.32  $ & $  73.53  \pm  0.48 $ &  $  75.50  \pm  3.04
$ &  $  75.34  \pm  2.96
$ &  $  75.66  \pm  2.45   $ & $  \textbf{77.13}  \pm  2.65  $ \\
wholesale  & $  91.02  \pm  2.87  $ & $  91.93  \pm  2.19  $ & $  90.80  \pm  2.28  $ & $  69.20  \pm  1.36  $ & $  90.80  \pm  2.12  $ & $  91.70  \pm  2.18  $ & $  69.09  \pm  1.26  $ & $  91.59  \pm  1.96  $ &  $  91.36  \pm  2.65
$ &  $  80.80  \pm  4.37
$ &  $  90.80  \pm  1.47  $ & $  \textbf{92.05}  \pm  2.25  $ \\
liver disorder  & $  72.76  \pm  5.61  $ & $  69.61  \pm  3.65  $ & $  72.76  \pm  5.45  $ & $  59.87  \pm  2.68  $ & $  71.50  \pm  5.32  $ & $  70.39  \pm  5.06  $ & $  66.94  \pm  4.71  $ & $  69.45  \pm  4.10  $ &  $  67.26  \pm  5.84
$ &  $  58.15  \pm  0.52
$ &  $  68.35  \pm  4.48  $ & $  \textbf{73.69}  \pm  4.69  $ \\
wine  & $  78.85  \pm  3.82  $ & $  \textbf{100.00}  \pm  0.00  $ & $  76.54  \pm  7.81  $ & $  69.23  \pm  6.81  $ & $  95.38  \pm  3.68  $ & $  \textbf{100.00}  \pm  0.00  $ & $  70.38  \pm  6.73  $ & $  99.62  \pm  0.77  $ &  $  99.23  \pm  1.54
$ &  $  93.85  \pm  4.62
$ &  $  99.62  \pm  1.09  $ & $  \textbf{100.00}  \pm  0.00  $ \\
balance  & $  99.74  \pm  0.52  $ & $  95.66  \pm  1.44  $ & $  99.39  \pm  0.49  $ & $  65.01  \pm  3.22  $ & $  99.13  \pm  0.93  $ & $  95.75  \pm  0.89  $ & $  61.89  \pm  2.51  $ & $  95.48  \pm  1.49  $ &  $  95.57  \pm  2.07
$ &  $  91.23  \pm  2.74
$ &  $  95.14  \pm  0.68  $ & $  \textbf{99.78}  \pm  0.22  $ \\
taxpayer  & $  \textbf{55.40}  \pm  3.56  $ & $  53.90  \pm  3.25  $ & $  53.67  \pm  1.67  $ & $  50.37  \pm  0.24  $ & $  53.97  \pm  3.23  $ & $  53.82  \pm  2.85  $ & $  51.28  \pm  2.44  $ & $  51.28  \pm  2.57  $ &  $  51.58  \pm  3.14
$ &  $  52.63  \pm  5.12
$ &  $  52.40  \pm  3.04  $ & $  54.65  \pm  3.28  $ \\
car evaluation  & $  95.14  \pm  0.73  $ & $  86.54  \pm  1.15  $ & $  94.42  \pm  0.93  $ & $  80.40  \pm  1.18  $ & $  92.10  \pm  0.94  $ & $  86.54  \pm  1.98  $ & $  81.15  \pm  1.54  $ & $  86.20  \pm  1.56  $ &  $  81.71  \pm  1.84
$ &  $  73.62  \pm  4.65
$ &  $  86.29  \pm  1.70  $ & $  \textbf{95.26}  \pm  0.92  $ \\
CTG  & $  98.38  \pm  0.55  $ & $  96.05  \pm  0.79  $ & $  98.02  \pm  0.62  $ & $  79.42  \pm  1.70  $ & $  97.34  \pm  0.63  $ & $  96.73  \pm  0.98  $ & $  77.85  \pm  0.07  $ & $  97.06  \pm  0.93 $ &  $  96.07  \pm  0.57
$ &  $  93.41  \pm  1.06
$ &  $  96.14  \pm  0.66   $ & $  \textbf{98.42}  \pm  0.63  $ \\
US heart  & $  84.99  \pm  0.80  $ & $  \textbf{85.35}  \pm  0.46  $ & $  84.83  \pm  0.31  $ & $  84.94  \pm  0.15  $ & $  84.94  \pm  0.21  $ & $  85.27  \pm  0.41  $ & $  84.75  \pm  0.18  $ & $  84.77  \pm  0.06  $ &  $  68.85  \pm  2.07
$ &  $  82.41  \pm  1.02
$ &  $  84.90  \pm  0.24  $ & $  85.21  \pm  0.48  $ \\
Alzheimer rowdel  & $  90.40  \pm  1.19  $ & $  90.12  \pm  1.72  $ & $  89.55  \pm  0.67  $ & $  86.30  \pm  2.59  $ & $  88.71  \pm  3.46  $ & $  90.40  \pm  1.68  $ & $  75.40  \pm  9.18  $ & $  89.55  \pm  0.67  $ &  $  71.03  \pm  5.49
$ &  $  89.55  \pm  0.67
$ &  $  89.84  \pm  1.17  $ & $  \textbf{90.68}  \pm  1.38  $ \\
Alzheimer clinical  & $  77.00  \pm  2.90  $ & $  76.67  \pm  2.28  $ & $  75.92  \pm  2.53  $ & $  70.88  \pm  1.23  $ & $  76.54  \pm  3.04  $ & $  75.88  \pm  2.32  $ & $  70.42  \pm  0.49  $ & $  76.71  \pm  2.29  $ &  $  72.08  \pm  3.46
$ &  $  64.50  \pm  1.76
$ &  $  73.71  \pm  2.22  $ & $  \textbf{77.42}  \pm  2.17  $ \\
thyroid  & $  92.09  \pm  5.22  $ & $  89.77  \pm  5.22  $ & $  91.63  \pm  4.56  $ & $  75.81  \pm  4.05  $ & $  92.09  \pm  4.31  $ & $  91.16  \pm  3.42  $ & $  90.23  \pm  5.58  $ & $  88.84  \pm  4.00  $ &  $  90.93  \pm  4.67
$ &  $  86.28  \pm  4.77
$ &  $  89.07  \pm  4.11  $ & $  \textbf{95.35}  \pm  2.55  $ \\
ecoli 1  & $  97.33  \pm  1.73  $ & $  97.33  \pm  1.73  $ & $  97.33  \pm  1.73  $ & $  67.55  \pm  2.94  $ & $  96.44  \pm  2.41  $ & $  97.03  \pm  1.88  $ & $  68.74  \pm  1.09  $ & $  96.73  \pm  2.19  $ &  $  97.18  \pm  2.25
$ &  $  97.17  \pm  2.62
$ &  $  97.32  \pm  2.06  $ & $  \textbf{97.92}  \pm  1.52  $ \\
ecoli 123  & $  91.38  \pm  1.92  $ & $  90.20  \pm  3.99  $ & $  90.78  \pm  1.69  $ & $  80.95  \pm  0.58  $ & $  90.18  \pm  2.57  $ & $  89.60  \pm  3.35  $ & $  80.95  \pm  0.58  $ & $  89.29  \pm  3.29  $ &  $  87.79  \pm  4.70
$ &  $  81.10  \pm  1.48
$ &  $  89.29  \pm  2.76  $ & $  \textbf{91.97}  \pm  2.20  $ \\
ecoli 4  & $  93.15  \pm  1.22  $ & $  93.15  \pm  0.75  $ & $  91.96  \pm  1.81  $ & $  81.27  \pm  4.67  $ & $  92.55  \pm  1.65  $ & $  92.84  \pm  2.58  $ & $  90.76  \pm  2.92  $ & $  92.85  \pm  1.48  $ &  $  88.98  \pm  2.71
$ &  $  85.56  \pm  3.29
$ &  $  92.72  \pm  2.22  $ & $  \textbf{94.04}  \pm  1.35  $ \\
ecoli 1234  & $  97.91  \pm  1.52  $ & $  97.62  \pm  1.52  $ & $  97.91  \pm  0.73  $ & $  91.37  \pm  0.59  $ & $  96.73  \pm  1.46  $ & $  97.03  \pm  1.62  $ & $  91.37  \pm  0.59  $ & $  97.62  \pm  1.20  $ &  $  95.99  \pm  2.38
$ &  $  92.11  \pm  0.59
$ &  $  97.62  \pm  1.64  $ & $  \textbf{98.21}  \pm  1.46  $ \\
ecoli 5  & $  \textbf{99.41}  \pm  0.73  $ & $  98.82  \pm  2.35  $ & $  98.81  \pm  1.10  $ & $  78.30  \pm  5.17  $ & $  98.51  \pm  0.01  $ & $  99.11  \pm  1.18  $ & $  94.05  \pm  2.50  $ & $  98.81  \pm  1.10  $ &  $  96.72  \pm  1.86
$ &  $  97.17  \pm  1.69
$ &  $  98.36  \pm  1.69  $ & $  \textbf{99.41}  \pm  1.18  $ \\
ecoli 678  & $  99.11  \pm  1.19  $ & $  99.11  \pm  0.96  $ & $  98.51  \pm  1.13  $ & $  97.32  \pm  0.59  $ & $  98.81  \pm  1.28  $ & $  \textbf{99.41}  \pm  0.96  $ & $  97.32  \pm  0.59  $ & $  99.26  \pm  0.96  $ &  $  98.51  \pm  1.16
$ &  $  98.51  \pm  1.16
$ &  $  99.26  \pm  0.99  $ & $  \textbf{99.41}  \pm  0.96  $ \\
\hline
Average rank of accuracy &  3.19 & 4.38 & 5.48 & 10.57 & 6.00 & 4.33 & 10.62 & 6.67 & 8.76 & 9.19 & 6.52 &  \textbf{1.29}
\\
\hline
\end{tabular}
}
\caption{Public benchmark data accuracy results.}
\label{table: benchmark data accuracy results}
\end{sidewaystable} %


\begin{sidewaystable} 
\centering
\resizebox{\textwidth}{!}{
\begin{tabular}{c|ccccccccccccc}\hline

\multirow{2}{*}{Dataset} & \multicolumn{12}{c}{G-mean (mean $\pm$ standard deviation) (\%)}  \\
\cline{2-13}
  &  LS-$\mathfrak{U}$-QTSVM      &  LS-$\mathfrak{U}$-TSVM      &  LS-QTSVM      &  LS-TSVM      &  $\mathfrak{U}$-QTSVM      &  $\mathfrak{U}$-TSVM      &  TSVM      &  SVM   & CSSVM & URKWELM & FHTPSVM &  Im-LS-$\mathfrak{U}$-QTSVM      \\
\hline

pima   & $   74.02   \pm   2.82   $ & $   75.27   \pm   3.49   $ & $   54.78   \pm   12.50   $ & $   46.47   \pm   3.38   $ & $   72.43   \pm   3.08   $ & $   71.53   \pm   4.81   $ & $   72.39   \pm   1.68   $ & $   70.41   \pm   3.91   $ & $   74.56   \pm   3.95   $ & $   74.01   \pm   0.50   $ & $   74.01   \pm   2.06   $ & $   \textbf{75.62}   \pm   2.31     $ \\ 
blood transfusion   & $   66.41   \pm   2.38   $ & $   69.76   \pm   2.20   $ & $   54.14   \pm   6.92   $ & $   32.58   \pm   5.55   $ & $   64.61   \pm   3.53   $ & $   60.52   \pm   3.33   $ & $   67.87   \pm   1.85   $ & $   30.21   \pm   9.56   $ & $   67.69   \pm   3.07   $ & $   67.18   \pm   3.27   $ & $   67.18   \pm   3.94   $ & $   \textbf{79.68}   \pm   1.24     $ \\ 
seismic bumps   & $   55.00   \pm   9.52   $ & $   62.26   \pm   4.99   $ & $   53.73   \pm   3.35   $ & $   38.36   \pm   3.29   $ & $   49.73   \pm   5.36   $ & $   46.15   \pm   4.33   $ & $   65.63   \pm   4.57   $ & $   11.62   \pm   16.13   $ & $   \textbf{71.15}   \pm   6.16   $ & $   \textbf{71.15}   \pm   1.97   $ & $   \textbf{71.15}   \pm   0.00   $ & $   70.51   \pm   6.77     $ \\ 
haberman   & $   63.84   \pm   3.90   $ & $   66.19   \pm   4.36   $ & $   44.00   \pm   10.27   $ & $   9.80   \pm   12.00   $ & $   59.86   \pm   4.45   $ & $   43.08   \pm   10.04   $ & $   64.22   \pm   6.50   $ & $   34.98   \pm   6.43   $ & $   57.26   \pm   6.71   $ & $   57.08   \pm   17.06   $ & $   57.08   \pm   15.71   $ & $   \textbf{66.41}   \pm   3.39     $ \\ 
wholesale   & $   89.23   \pm   3.22   $ & $   90.78   \pm   3.85   $ & $   52.39   \pm   26.98   $ & $   43.38   \pm   8.60   $ & $   87.98   \pm   2.36   $ & $   90.70   \pm   3.86   $ & $   56.46   \pm   13.74   $ & $   85.17   \pm   5.20   $ & $   91.22   \pm   2.97   $ & $   91.22   \pm   4.96   $ & $   91.22   \pm   3.78   $ & $   \textbf{91.54}   \pm   3.73     $ \\ 
liver disorder   & $   71.04   \pm   8.11   $ & $   65.02   \pm   4.47   $ & $   52.14   \pm   6.39   $ & $   29.22   \pm   17.41   $ & $   70.62   \pm   6.95   $ & $   65.72   \pm   5.81   $ & $   51.66   \pm   8.09   $ & $   66.06   \pm   4.48   $ & $   66.96   \pm   3.94   $ & $   66.81   \pm   0.00   $ & $   66.81   \pm   7.04   $ & $   \textbf{72.90}   \pm   5.15     $ \\ 
wine   & $   90.15   \pm   4.80   $ & $   \textbf{100.00}   \pm   0.00   $ & $   32.57   \pm   6.24   $ & $   61.90   \pm   10.11   $ & $   96.31   \pm   2.20   $ & $   \textbf{100.00}   \pm   0.00   $ & $   62.79   \pm   17.86   $ & $   98.59   \pm   1.72   $ & $   99.27   \pm   1.46   $ & $   97.87   \pm   0.00   $ & $   97.87   \pm   3.14   $ & $   \textbf{100.00}   \pm   0.00     $ \\ 
balance   & $   99.83   \pm   0.35   $ & $   94.08   \pm   1.85   $ & $   78.31   \pm   2.78   $ & $   54.13   \pm   7.19   $ & $   99.13   \pm   0.78   $ & $   95.58   \pm   0.79   $ & $   48.35   \pm   6.69   $ & $   94.95   \pm   2.01   $ & $   95.46   \pm   2.09   $ & $   94.95   \pm   0.30   $ & $   94.95   \pm   2.19   $ & $   \textbf{100.00}   \pm   0.00     $ \\ 
taxpayer   & $   53.52   \pm   2.10   $ & $   50.72   \pm   1.49   $ & $   46.23   \pm   4.16   $ & $   41.81   \pm   6.53   $ & $   52.96   \pm   2.25   $ & $   50.62   \pm   2.48   $ & $   47.40   \pm   8.37   $ & $   49.69   \pm   3.10   $ & $   50.91   \pm   2.89   $ & $   49.49   \pm   12.06   $ & $   49.49   \pm   3.74   $ & $   \textbf{55.52}   \pm   3.24     $ \\ 
car evaluation   & $   95.14   \pm   1.23   $ & $   75.37   \pm   4.04   $ & $   58.79   \pm   4.44   $ & $   52.95   \pm   2.45   $ & $   94.38   \pm   1.24   $ & $   79.41   \pm   3.56   $ & $   82.45   \pm   1.98   $ & $   77.96   \pm   3.98   $ & $   82.78   \pm   2.34   $ & $   82.78   \pm   0.63   $ & $   82.78   \pm   10.00   $ & $   \textbf{95.35}   \pm   0.74     $ \\ 
CTG   & $   \textbf{97.65}   \pm   1.08   $ & $   90.59   \pm   1.92   $ & $   86.06   \pm   4.91   $ & $   61.72   \pm   3.05   $ & $   96.92   \pm   0.81   $ & $   93.44   \pm   1.81   $ & $   49.02   \pm   17.29   $ & $   95.13   \pm   1.83   $ & $   95.67   \pm   0.94   $ & $   95.49   \pm   0.67   $ & $   95.49   \pm   1.58   $ & $   97.44   \pm   1.45     $ \\ 
US heart   & $   52.21   \pm   2.58   $ & $   44.54   \pm   3.64   $ & $   34.38   \pm   13.70   $ & $   36.48   \pm   3.34   $ & $   50.23   \pm   4.28   $ & $   48.21   \pm   6.03   $ & $   60.49   \pm   2.04   $ & $   56.41   \pm   6.30   $ & $   66.95   \pm   3.18   $ & $   66.79   \pm   3.40   $ & $   66.79   \pm   2.53   $ & $   {67.43}   \pm   3.31     $ \\ 
Alzheimer rowdel   & $   76.49   \pm   2.82   $ & $   73.29   \pm   5.66   $ & $   28.65   \pm   20.21   $ & $   42.26   \pm   9.06   $ & $   73.86   \pm   5.36   $ & $   52.75   \pm   10.29   $ & $   70.67   \pm   3.03   $ & $   27.67   \pm   23.08   $ & $   75.11   \pm   5.32   $ & $   75.11   \pm   8.55   $ & $   75.11   \pm   0.00   $ & $   \textbf{77.45}   \pm   5.36     $ \\ 
Alzheimer clinical   & $   77.43   \pm   2.61   $ & $   \textbf{78.43}   \pm   3.43   $ & $   56.89   \pm   11.68   $ & $   35.72   \pm   4.18   $ & $   77.52   \pm   1.34   $ & $   61.05   \pm   17.09   $ & $   53.87   \pm   3.15   $ & $   72.64   \pm   4.60   $ & $   76.85   \pm   3.78   $ & $   76.85   \pm   3.49   $ & $   76.85   \pm   0.00   $ & $   78.37   \pm   2.47     $ \\ 
thyroid   & $   90.23   \pm   3.61   $ & $   85.78   \pm   4.60   $ & $   82.40   \pm   6.36   $ & $   39.19   \pm   21.55   $ & $   86.07   \pm   8.64   $ & $   86.89   \pm   5.95   $ & $   82.41   \pm   12.77   $ & $   80.38   \pm   7.52   $ & $   87.38   \pm   6.50   $ & $   87.38   \pm   15.70   $ & $   87.38   \pm   6.43   $ & $   \textbf{91.41}   \pm   4.26     $ \\ 
ecoli 1   & $   97.39   \pm   1.19   $ & $   97.62   \pm   2.02   $ & $   97.33   \pm   2.18   $ & $   54.17   \pm   5.58   $ & $   96.06   \pm   2.35   $ & $   97.02   \pm   1.21   $ & $   51.47   \pm   2.40   $ & $   96.64   \pm   2.13   $ & $   97.36   \pm   2.37   $ & $   97.36   \pm   1.80   $ & $   97.36   \pm   4.46   $ & $   \textbf{97.95}   \pm   1.80     $ \\ 
ecoli 123   & $   88.74   \pm   3.53   $ & $   \textbf{90.24}   \pm   3.45   $ & $   43.32   \pm   13.54   $ & $   57.77   \pm   8.98   $ & $   82.25   \pm   6.08   $ & $   86.06   \pm   6.16   $ & $   47.82   \pm   7.74   $ & $   79.41   \pm   7.71   $ & $   89.88   \pm   3.93   $ & $   88.41   \pm   0.00   $ & $   88.41   \pm   11.42   $ & $   \textbf{89.73}   \pm   4.36     $ \\ 
ecoli 4   & $   75.41   \pm   14.30   $ & $   88.02   \pm   7.56   $ & $   84.73   \pm   6.13   $ & $   59.30   \pm   6.12   $ & $   78.80   \pm   7.05   $ & $   78.49   \pm   4.61   $ & $   90.09   \pm   5.35   $ & $   71.67   \pm   6.63   $ & $   87.66   \pm   7.30   $ & $   86.88   \pm   2.38   $ & $   86.88   \pm   28.25   $ & $   \textbf{91.37}   \pm   4.21     $ \\ 
ecoli 1234   & $   89.83   \pm   9.61   $ & $   83.62   \pm   10.94   $ & $   57.07   \pm   14.60   $ & $   53.66   \pm   10.18   $ & $   87.22   \pm   10.43   $ & $   89.69   \pm   9.73   $ & $   60.94   \pm   21.23   $ & $   87.13   \pm   13.96   $ & $   92.40   \pm   9.47   $ & $   87.04   \pm   0.00   $ & $   87.04   \pm   22.52   $ & $   \textbf{95.91}   \pm   4.68     $ \\ 
ecoli 5   & $   91.69   \pm   6.79   $ & $   89.68   \pm   5.60   $ & $   64.19   \pm   19.36   $ & $   62.50   \pm   8.89   $ & $   93.90   \pm   6.87   $ & $   91.46   \pm   11.60   $ & $   49.83   \pm   21.56   $ & $   94.18   \pm   6.50   $ & $   97.75   \pm   0.96   $ & $   93.13   \pm   4.21   $ & $   93.13   \pm   30.28   $ & $   \textbf{98.40}   \pm   0.88     $ \\ 
ecoli 678   & $   88.13   \pm   14.23   $ & $   88.28   \pm   14.35   $ & $   72.25   \pm   15.76   $ & $   38.28   \pm   46.89   $ & $   74.03   \pm   38.78   $ & $   88.28   \pm   14.35   $ & $   62.45   \pm   8.63   $ & $   88.28   \pm   14.35   $ & $   87.87   \pm   14.24   $ & $   84.95   \pm   0.00   $ & $   84.95   \pm   38.67   $ & $   \textbf{92.80}   \pm   11.43     $ \\ 
\hline
Average rank of G-mean &  4.48 &  5.33 &  10.14 &  11.43 &  6.10 &  7.14 &  8.67 &  8.29 &  3.52 &  4.81 &  4.81 &    \textbf{1.33}
\\
\hline
\end{tabular}
}
\caption{\textcolor{black}{Public benchmark data G-mean results.}}
\label{table: benchmark data G-mean results}
\end{sidewaystable} %

\begin{sidewaystable} %
\centering
\resizebox{\textwidth}{!}{
\begin{tabular}{c|ccccccccBBBc}\hline

\multirow{2}{*}{Dataset} & \multicolumn{12}{c}{CPU time (s)}  \\
\cline{2-13}
  &  LS-$\mathfrak{U}$-QTSVM      &  LS-$\mathfrak{U}$-TSVM      &  LS-QTSVM      &  LS-TSVM      &  $\mathfrak{U}$-QTSVM      &  $\mathfrak{U}$-TSVM      &  TSVM      &  SVM   & CSSVM & URKWELM & FHTPSVM &  Im-LS-$\mathfrak{U}$-QTSVM      \\
\hline
pima & 0.141 & 0.013 & 0.128 & 0.181 & 0.537 & 0.671 & 0.798 & 0.035 & 0.045 & 1.205 & 0.233 & 0.199 \\
blood transfusion & 0.053 & 0.013 & 0.047 & 0.184 & 0.664 & 0.993 & 1.025 & 0.027 & 0.044 & 0.668 & 0.635 & 0.106 \\
seismic bumps & 1.690 & 0.079 & 1.555 & 4.766 & 41.329 & 33.805 & 54.145 & 0.126 & 0.484 & 4.742 & 1.940 & 2.847 \\
haberman & 0.019 & 0.008 & 0.016 & 0.043 & 0.097 & 0.152 & 0.178 & 0.009 & 0.013 & 0.112 & 0.071 & 0.032 \\
wholesale & 0.067 & 0.007 & 0.062 & 0.066 & 0.205 & 0.198 & 0.257 & 0.010 & 0.016 & 0.243 & 0.303 & 0.088 \\
liver disorder & 0.038 & 0.007 & 0.034 & 0.048 & 0.099 & 0.121 & 0.057 & 0.011 & 0.014 & 0.060 & 0.094 & 0.047 \\
wine & 0.065 & 0.004 & 0.065 & 0.008 & 0.082 & 0.021 & 0.019 & 0.004 & 0.007 & 0.070 & 0.013 & 0.070 \\
balance & 0.040 & 0.010 & 0.034 & 0.104 & 0.277 & 0.408 & 0.324 & 0.007 & 0.014 & 0.292 & 0.175 & 0.031 \\
taxpayer & 0.150 & 0.012 & 0.140 & 0.123 & 0.417 & 0.480 & 0.475 & 0.043 & 0.044 & 0.214 & 0.078 & 0.156 \\
car evaluation & 0.195 & 0.039 & 0.164 & 1.031 & 3.272 & 3.861 & 5.591 & 0.088 & 0.106 & 1.326 & 0.833 & 0.345 \\
CTG & 3.303 & 0.059 & 3.126 & 2.399 & 9.994 & 11.610 & 2.344 & 0.070 & 0.106 & 2.705 & 1.113 & 4.118 \\
US heart & 2.719 & 0.137 & 2.211 & 10.382 & 59.138 & 72.856 & 103.524 & 0.400 & 1.091 & 11.199 & 3.571 & 3.661 \\
Alzheimer rowdel & 0.081 & 0.007 & 0.077 & 0.056 & 0.230 & 0.252 & 0.243 & 0.011 & 0.014 & 0.091 & 0.231 & 0.141 \\
Alzheimer clinical & 0.114 & 0.024 & 0.096 & 0.590 & 1.965 & 2.174 & 2.189 & 0.070 & 0.081 & 1.013 & 0.691 & 0.178 \\
thyroid & 0.022 & 0.005 & 0.020 & 0.020 & 0.072 & 0.071 & 0.054 & 0.006 & 0.011 & 0.201 & 0.048 & 0.042 \\
ecoli 1 & 0.052 & 0.007 & 0.050 & 0.045 & 0.140 & 0.130 & 0.119 & 0.005 & 0.009 & 0.103 & 0.094 & 0.061 \\
ecoli 123 & 0.050 & 0.007 & 0.046 & 0.061 & 0.158 & 0.127 & 0.040 & 0.006 & 0.012 & 0.082 & 0.052 & 0.087 \\
ecoli 4 & 0.051 & 0.007 & 0.053 & 0.035 & 0.176 & 0.274 & 0.219 & 0.005 & 0.010 & 0.142 & 0.023 & 0.099 \\
ecoli 1234 & 0.052 & 0.006 & 0.046 & 0.025 & 0.186 & 0.155 & 0.062 & 0.005 & 0.010 & 0.093 & 0.252 & 0.096 \\
ecoli 5 & 0.053 & 0.007 & 0.048 & 0.031 & 0.172 & 0.198 & 0.246 & 0.005 & 0.009 & 0.125 & 0.375 & 0.094 \\
ecoli 678 & 0.051 & 0.007 & 0.049 & 0.049 & 0.239 & 0.205 & 0.054 & 0.004 & 0.011 & 0.229 & 0.031 & 0.097 \\
\hline
\end{tabular}
}
\caption{CPU time (s) consumed by implementing the model with fixed parameters.}
\label{table: benchmark data CPU time results}
\end{sidewaystable} 

For each tested model, we record the mean and the standard deviation of the classification accuracy and G-mean on each data set in Tables \ref{table: benchmark data accuracy results} and \ref{table: benchmark data G-mean results}, respectively. The average ranks of accuracy and G-mean scores on all benchmark datasets are also calculated for each tested model. Moreover, the highest accuracy and G-mean scores on each benchmark dataset are highlighted. In addition, we have recorded the CPU time consumed by each tested model on all the benchmark datasets in Table \ref{table: benchmark data CPU time results}.

\begin{figure}[t]
  \centering
\includegraphics[width=.9\textwidth]{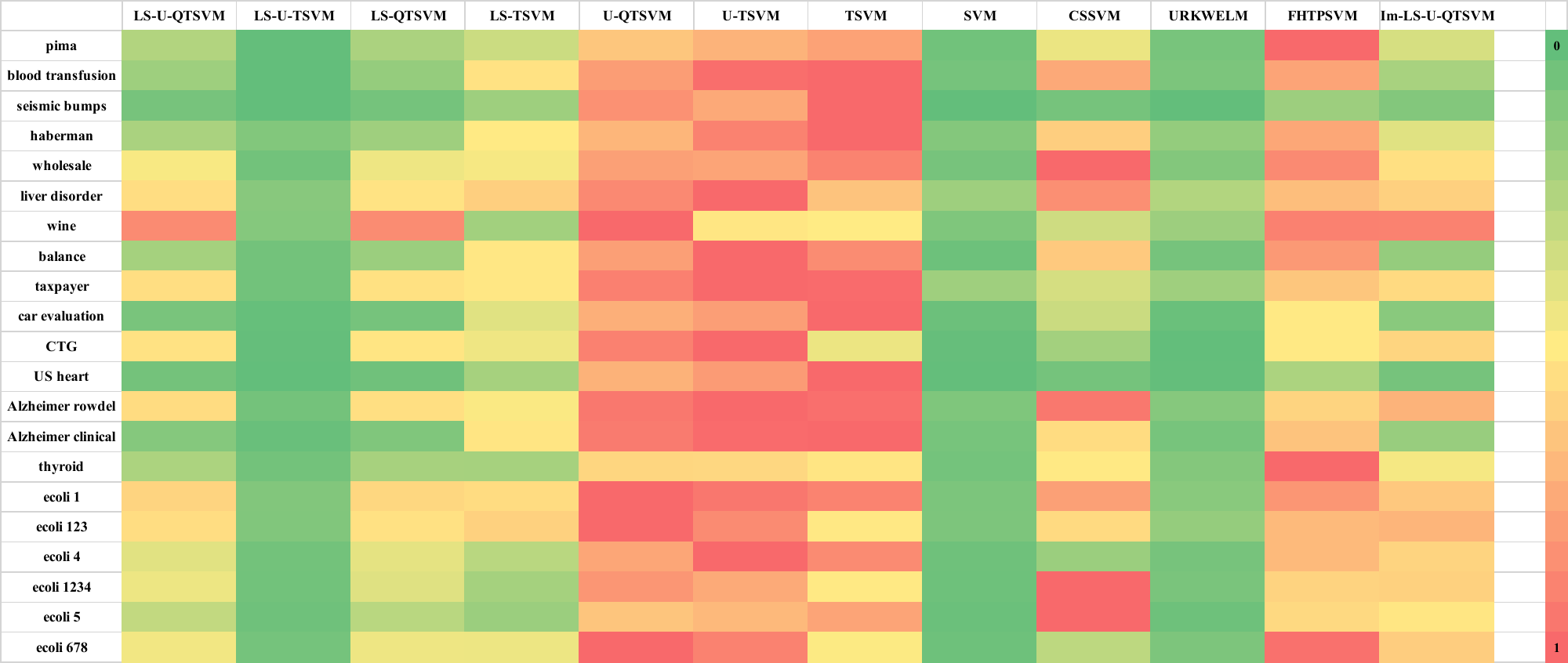}
  \caption{\textcolor{black}{Heatmap of relative rank of CPU time. For each row, the unit is closer to red if its CPU time is longer, while the unit is closer to green if the CPU time is shorter.}}
  \label{figure: heatmap of CPU time}
\end{figure}

The observations from the results in Tables \ref{table: benchmark data accuracy results} and \ref{table: benchmark data CPU time results} can be summarized as follows.

\begin{itemize}
    \item Compared with all the other tested models, the proposed Im-LS-$\mathfrak{U}$-QTSVM achieves the highest mean accuracy and G-mean scores on most of the benchmark datasets. Besides, it provides the highest average rank of both accuracy and G-mean score. In other words, the proposed model provides a better general classification accuracy than other tested models do in this experiment.

    \item In general, the CPU time consumed by the Im-LS-$\mathfrak{U}$-QTSVM model is acceptable. As a least squares model, the proposed Im-LS-$\mathfrak{U}$-QTSVM does not lose efficiency when the size of data set increases. Even for the relatively large-scale datasets \textit{seismic bumps}, \textit{US heart}, and \textit{CTG}, the CPU time consumed by the proposed model is acceptable, and is less than that consumed by the other twin SVM models without least squares. 

    \item Furthermore, we plot a heatmap in Figure \ref{figure: heatmap of CPU time} to show the comparison of efficiency among all the tested models on each dataset. We may observe that, in most cases, the relative rank of the CPU time consumed by the proposed Im-LS-$\mathfrak{U}$-QTSVM model stays at the medium level among all the tested models. Compared with other tested models on a given dataset, the proposed model might be more efficient if the ratio $N/n$ is higher, where $n$ is the number of features and $N$ is the number of data samples. 
\end{itemize}

\subsection{Statistical analysis}
\label{subsection: statistical analysis}

Some statistical tests are conducted to further analyze the computational results we received in Section \ref{subsection: numerical section public benchmark datasets}. We first conduct the Friedman test \cite{demvsar2006statistical} to evaluate the classification accuracy in Table \ref{table: benchmark data accuracy results}. The null hypothesis is that all the tested models have the same classification accuracy.

Let $p$ and $q$ be the number of datasets and the number of tested models, respectively. Hence, $p = 21$ and $q = 12$, and the $\chi_F^2$ value of the Friedman test can be calculated by the formula (\ref{eq: Friedman test chi^2_F})

\begin{equation}
    \label{eq: Friedman test chi^2_F}
    \chi_F^2 = \frac{12p}{q(q+1)} 
    \left( \sum_{i=1}^q R_i^2 - \frac{q(q+1)^2}{4} \right),
\end{equation}
where $R_i$ is the average rank of accuracy by the $i$th model. And the statistics $F_F$ is then calculated in (\ref{eq: Friedman test F_F})

\begin{equation}
    \label{eq: Friedman test F_F}
    F_F = \frac{(p-1)\chi_F^2}{p(q-1) - \chi_F^2}, 
\end{equation}
where $F_F$ is distributed according to the $F$-distribution with degrees of freedom $(q-1, (p-1)\times(q-1)) = (11, 220)$. \textcolor{black}{Similarly, the Friedman test is applied to evaluate the G-mean recorded in Table \ref{table: benchmark data G-mean results} using formulas \eqref{eq: Friedman test chi^2_F} and \eqref{eq: Friedman test F_F}. The key parameters of Friedman tests on accuracy and G-mean are listed in Table \ref{table: Friedman test parameters}. With the significance level to be $\alpha = 0.05$, the critical value of $F_F$ is about 2.03, which is far less than the $F_F$ of either accuracy or G-mean. Therefore, we are confident in rejecting the null hypothesis. In other words, there exists a significant difference among all the tested models with respect to classification accuracy and G-mean.}

\begin{table}[h]
  \centering
    \begin{tabular}{l l l}
  \hline Metrics & $\chi_F^2$ & $F_F$ \\
  \hline
    Accuracy & 109.385 &  17.990 \\
    G-mean & 106.308 & 17.051 \\
    \hline
    \end{tabular}
    
  \caption{\textcolor{black}{Key parameters of Friedman test.}}
  \label{table: Friedman test parameters}
\end{table}  

{
\color{black}
To further analyze the difference, the post-hoc Nemenyi test is performed as follows. The statistics \textit{CD} is calculated by (\ref{eq: Nemenyi CD value})

\begin{equation}
    \label{eq: Nemenyi CD value}
    CD = q_{.05} \sqrt{\frac{q(q+1)}{6p}} \approx 3.636,
\end{equation}
where $q_{.05} = 3.268$ is the critical value of the Tukey distribution. The results from the Nemenyi post-hoc test on both accuracy and G-mean are visualized in Figure \ref{figure: Nemenyi test}.
}

\begin{figure}
    \centering
    \begin{subfigure}[b]{0.48\textwidth}
        \includegraphics[width=\textwidth]{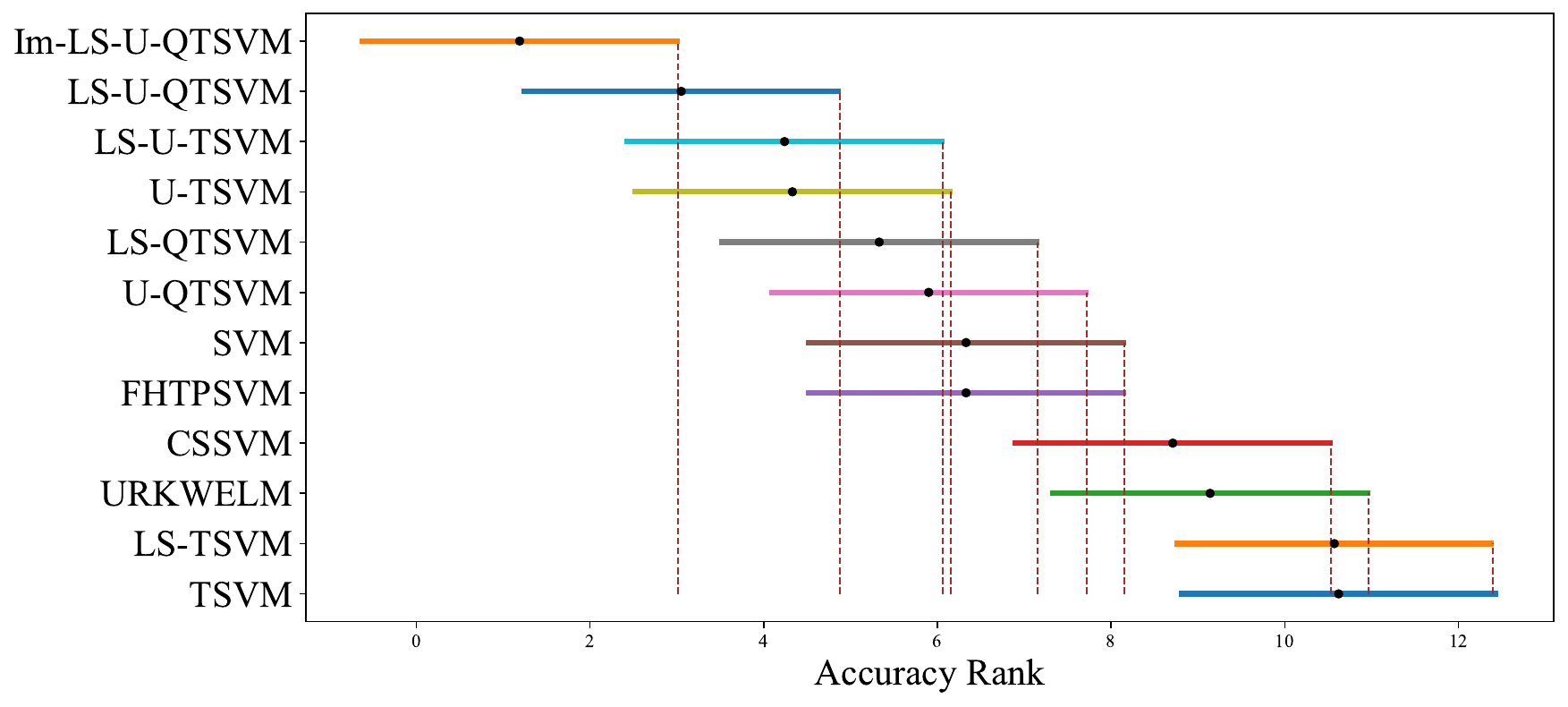}
          \caption{Nemenyi test on Accuracy.}
  \label{figure: Nemenyi_acc}
    \end{subfigure}
    ~ 
    \begin{subfigure}[b]{0.48\textwidth}
        \includegraphics[width=\textwidth]{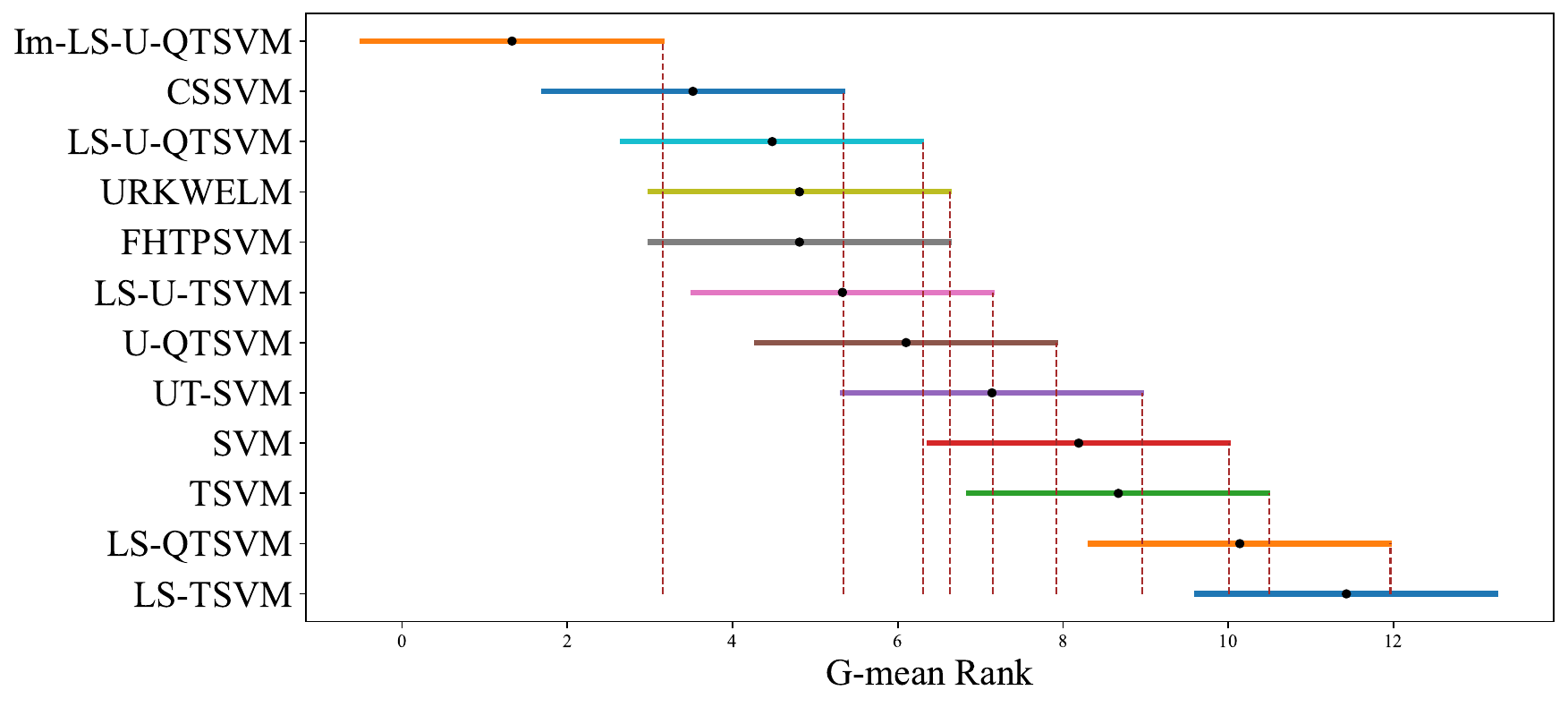}
          \caption{Nemenyi test on G-mean.}
  \label{figure: Nemenyi_gm}
    \end{subfigure}
    \caption{\textcolor{black}{Visualization of Nemenyi test.}}\label{figure: Nemenyi test}
\end{figure}

{\color{black}

From Figure \ref{figure: Nemenyi_acc}, we notice that the average rank in accuracy of Im-LS-$\mathfrak{U}$-QTSVM is close to LS-$\mathfrak{U}$-QTSVM. From Figure \ref{figure: Nemenyi_acc}, we notice that the average rank in G-mean of Im-LS-$\mathfrak{U}$-QTSVM is close to CSSVM. Therefore, we further apply the Wilcoxon signed-rank test as suggested in \cite{benavoli2016should} to show the difference between the Im-LS-$\mathfrak{U}$-QTSVM model and each other tested model. For each metric, the null hypothesis is that there is no significant difference between the two models. Let $\beta_j$ be the difference between the metric of the two models on the $j$th dataset. The calculated values $R^+$ and $R^-$ defined as follows
\begin{align*}   
R^+ &= \sum_{\beta_i > 0} \text{rank}(\beta_i) + \frac{1}{2}\sum_{\beta_i = 0} \text{rank}(\beta_i), \\
R^- &= \sum_{\beta_i < 0} \text{rank}(\beta_i) + \frac{1}{2}\sum_{\beta_i = 0} \text{rank}(\beta_i),
\end{align*}
Let $R_{\min}  = \min \{ R^+, R^- \}$ and calculate the statistics $z$:
$$
z = \frac{R_{\min} - \frac{1}{4} p(p+1)}{\sqrt{\frac{1}{24} p(p+1)(2p+1)}}.
$$
For each metric, we calculate the $z$ value between the Im-LS-$\mathfrak{U}$-QTSVM model and each benchmark model. All the $z$ values are recorded in Table \ref{table: Wilcoxon test z values}. Notice that all the $z$-statistic values are less than $-1.96$, which leads to the rejection of the null hypothesis with a $0.05$ confidence level.

In conclusion, the proposed Im-LS-$\mathfrak{U}$-QTSVM model significantly outperforms all the other tested models with respect to classification accuracy on the benchmark datasets in Table \ref{table: pubic benchmark data info}.
}

\begin{table}[h]
  \centering
    \begin{tabular}{l l l}
  \hline Compared Models & Accuracy & G-mean \\
  \hline
LS-$\mathfrak{U}$-QTSVM  &   -3.15  &  -3.74  \\
LS-$\mathfrak{U}$-TSVM  &   -3.53  &   -2.94  \\
LS-QTSVM  &   -4.01  &   -4.01  \\
LS-TSVM  &   -4.01  &   -4.01  \\
$\mathfrak{U}$-QTSVM  &   -4.01  &   -4.01  \\
$\mathfrak{U}$-TSVM  &   -2.80  &   -3.89  \\
TSVM  &   -4.01  &   -4.01  \\
SVM  &   -3.91  &   -4.01  \\
CSSVM  &   -4.01  &   -2.94  \\
URKWELM  &   -3.32  &   -3.94  \\
FHTPSVM  &   -3.82  &   -3.94  \\
    \hline
    \end{tabular}
    
  \caption{\textcolor{black}{(Wilcoxon signed-rank test) $z$-statistic between Im-LS-$\mathfrak{U}$-QTSVM and each compared model.}}
  \label{table: Wilcoxon test z values}
\end{table}  
}

\section{Conclusions { and Future Research Directions}}
\label{sec: Conclusions}

In this paper, we introduced three kernel-free Universum quadratic surface support vector machine models for binary classification: $\mathfrak{U}$-QTSVM, LS-$\mathfrak{U}$-QTSVM, and Im-LS-$\mathfrak{U}$-QTSVM. { Binary classification with imbalanced classes remains a major challenge in machine learning, especially when minority class instances are scarce. Traditional classifiers often exhibit bias toward the majority class, leading to poor prediction performance on the minority class and potentially serious consequences in real-world applications such as medical diagnosis, fraud detection, or rare-event prediction.} To address this issue, we developed the Im-LS-$\mathfrak{U}$-QTSVM model, tailored for imbalanced datasets. We examined several theoretical properties and proposed an efficient algorithm for this model. Our computational experiments confirmed the effectiveness and efficiency of the proposed models. The main findings are summarized below:

\begin{itemize}
    \item { We proposed the $\mathfrak{U}$-QTSVM model, which leverages Universum points to provide additional prior knowledge that supports minority class representation. This approach enhances the classifier’s ability to capture subtle patterns in complex datasets, improving generalization performance without requiring kernel-based transformations. Its least-squares version, LS-$\mathfrak{U}$-QTSVM, further simplifies the optimization problem by converting it into a system of linear equations, making it computationally efficient for moderate-scale datasets.}
    
    \item To tackle the challenge of imbalanced datasets, we adapted $\mathfrak{U}$-QTSVM to create Im-$\mathfrak{U}$-QTSVM and its least-square version, Im-LS-$\mathfrak{U}$-QTSVM, which reduces computational complexity by solving a system of equations. { Indeed, these models explicitly focus on improving classification of the minority class by adjusting the margin and leveraging Universum points strategically, which reduces the bias toward the majority class. The least-squares formulation reduces computational complexity while maintaining strong predictive performance, making it suitable for real-world imbalanced datasets.}
    
    \item We explored several theoretical properties of the Im-LS-$\mathfrak{U}$-QTSVM model.
    
    \item Numerical experiments demonstrated the impact of parameters on classification accuracy. The promising results on various artificial and public benchmark datasets indicate the effectiveness of the proposed models in addressing real-world binary classification problems.
\end{itemize}

{ Building on this work, future research could extend the proposed models to multi-class imbalanced datasets to handle more complex classification scenarios. Exploring adaptive selection or generation of Universum points may further improve generalization and reduce the need for manual parameter tuning. Developing efficient solvers or approximation methods for high-dimensional or large-scale datasets would enhance the models’ practical applicability. Immediate future work also includes assessing the robustness of the proposed models across diverse binary classification tasks and exploring extensions to semi-supervised learning settings.

In summary, the proposed Universum quadratic surface SVM models provide a flexible, computationally efficient, and effective solution for imbalanced binary classification. By addressing the limitations of traditional SVMs and explicitly considering minority class support, these models open up opportunities for reliable deployment in real-world applications. The future directions outlined above provide a roadmap for enhancing the generalizability, scalability, and robustness of these models, making them even more compelling for researchers and practitioners.}

\section*{Compliance with Ethical Standards}
\begin{itemize}
    \item \textbf{Conflict of interest:} The authors declare that they have no conflicts of interest.
    \item \textbf{Ethical approval:} This article does not contain any studies with human participants or animals performed by any of the authors.
    \item \textbf{Informed consent:} Informed consent was obtained from all individual participants included in the study.
    \item \textbf{Data availability statement:} The following GitHub repository provides all the links to the publicly available datasets that have been used for the numerical experiments of this study: \texttt{https://github.com/tonygaobasketball/Sparse-UQSSVM-Mode} \linebreak
\texttt{ls-for-Binary-Classification}. The code will be provided by the authors upon request.
\end{itemize}
\textbf{Authorship contributions:} 
\textbf{Hossein Moosaei:} Conceptualization, Formal analysis, Investigation, Methodology, Project administration,  Writing – review \& editing. \textbf{Milan Hlad\'{i}k:} Conceptualization, Formal analysis,  Project administration, Supervision, Validation,  Writing – review \& editing. \textbf{Ahmad Mousavi:}  Formal analysis, Investigation, Writing, and Editing.  \textbf{Zheming Gao:} Conceptualization, Data curation, Formal analysis, Investigation, Methodology, Software, Validation, Supervision, Writing – original draft. \textbf{Haojie Fu:} Software, validation, Writing – original draft.


\medskip
\medskip

\appendix

\end{document}